\def\comments{0}

\documentclass[11pt]{article}

\usepackage[utf8]{inputenc} 
\usepackage[T1]{fontenc}    
\usepackage{url}            
\usepackage{booktabs}       
\usepackage{amsfonts}       
\usepackage{nicefrac}       
\usepackage{microtype}      
\usepackage{xcolor}         
\usepackage[margin=1.2in]{geometry}
\usepackage{mymacros}

\usepackage{amsmath,amssymb,amsthm}
\usepackage[shortlabels]{enumitem}
\usepackage{soul} 

\usepackage{authblk}

\usepackage[ruled,vlined]{algorithm2e}
\SetKwInput{KwInput}{Input}
\SetKwInput{KwOutput}{Output}

\usepackage[colorlinks,citecolor=blue,urlcolor=blue,linkcolor=blue]{hyperref}

\ifnum\comments=1
\newcommand{\mynote}[2]{{\marginpar{\color{#1} \tiny #2}}}
\else
\newcommand{\mynote}[2]{}
\fi
\newcommand{\anote}[1]{\mynote{cyan}{Alex: {#1}}}
\newcommand{\snote}[1]{\mynote{red}{Sasho: {#1}}}
\newcommand{\tnote}[1]{\mynote{blue}{Toni: {#1}}}

\newcounter{listcounter}

\newtheorem{theorem}{Theorem}

\newtheorem{lemma}[theorem]{Lemma}

\newtheorem{corollary}[theorem]{Corollary}

\newtheorem{definition}[theorem]{Definition}

\newcommand{\cut}[1]{}

\title{Learning versus Refutation in \\ Noninteractive Local Differential Privacy}

\author[1]{Alexander Edmonds}
\author[1]{Aleksandar Nikolov}
\author[1,2]{Toniann Pitassi}
\affil[1]{Department of Computer Science, University of Toronto}
\affil[2]{Department of Computer Science, Columbia University}
\date{}

\begin{document}

\maketitle

\begin{abstract}

  We study two basic statistical tasks in non-interactive local differential privacy (LDP): {\it learning} and {\it refutation}; learning requires finding a concept that best fits an unknown target function (from labelled samples drawn from a distribution), whereas  refutation requires distinguishing between data distributions that are well-correlated with some concept in the class, versus distributions where the labels are random. Our main result is a complete characterization of the sample complexity of agnostic PAC learning for non-interactive LDP protocols. We show that the optimal sample complexity for any concept class is captured by the approximate $\gamma_2$~norm of a natural matrix associated with the class. Combined with previous work [Edmonds, Nikolov and Ullman, 2019] this gives an {\it equivalence} between learning and refutation in the agnostic setting.

\end{abstract}


\section{Introduction}

We study two related basic statistical tasks, \emph{learning} and \emph{refutation}, in the setting of distributed data, and under strong privacy constraints. For both tasks, we have an unknown distribution
$\dista$ on labeled data points in the universe $\iuni \times \oo$, and we receive samples from \(\dista\). We are also given a concept class $\concepts \subseteq \oo^\iuni$, which, hopefully, is capable of capturing the labels given by $\dista$. We define our two tasks as follows.

\begin{itemize}
    \item \emph{Learning} requires finding a concept that best fits \(\dista\). I.e., using the usual binary loss function 
    $\loss_\dista(\hyp)
    = \ex{(\ielem,\lab) \sim \dista}{\I[\hyp(\ielem) \ne y]}$,
    the goal of agnostic learning with accuracy $\alpha$ is to produce some $\hyp$ which, with probability \(1-\conf\), satisfies
    $\loss_\dista(\hyp)
    \le \min_{\concept \in \concepts}
        \loss_\dista(\concept) + \alpha$.
    
    If an algorithm solves this problem for any distribution \(\dista\), then we say it \((\acc,\conf)\)-learns \(\concepts\) agnostically. 

    \item \emph{Refutation} requires distinguishing between data distributions \(\dista\) that are well correlated with some concept \(\concept \in \concepts\), vs.~data distributions where the labels are random. I.e., the goal of agnostic refutation with accuracy \(\alpha\) is to distinguish, with probability \(1-\conf\), between the following two cases:
    (i) $\min_{\concept \in \concepts}\loss_{\dista}(\concept) \le \frac12 - \alpha$
    versus (ii)
    for all $\hyp \in \oo^\iuni$,
            $\loss_\dista(\hyp) = \frac12$.
        
    If an algorithm solves this problem for any distribution \(\dista\), then we say it \((\acc,\conf)\)-refutes \(\concepts\) agnostically. 
\end{itemize}

The definition of agnostic learning above is classical. Refutation is
a more recent notion, and was studied by \cite{KL18} (and in a
realizable setting by \cite{learningvsrefutation}), where it was shown
that computationally efficient refutation is equivalent to
computationally efficient agnostic learning. Refutation is a testing version of the problem of evaluating the choice of model in supervised learning, i.e., of estimating the best achievable loss \(\min_{\concept \in \concepts}\loss_{\dista}(\concept)\) by the concept class \(\concepts\). While agnostic learning is well-defined for any concept class, it is less meaningful when the best achievable loss is trivially large, which may be an indication that we need to choose a different model, i.e., a different concept class. For this reason, ideally we would like our learning algorithm to also tell us what loss it is able to achieve. Refutation is a more basic version of this problem, in which we merely want to distinguish data distributions for which our model is good from distributions with random labels, for which no model can achieve good results. Certainly being able to solve the refutation problem is at least as hard as estimating \(\min_{\concept \in \concepts}\loss_{\dista}(\concept)\).

In this paper, we study learning and refutation in the model of \emph{non-interactive local differential privacy} (LDP)~\cite{KasiviswanathanLNRS08}. LDP applies in a distributed setting in which each data point represents one person, and, in order to protect privacy, the person retains ownership of their data point. In particular, the data is never centrally collected, and, instead, the data owners communicate differentially private randomized message to a central server. The differential privacy~\cite{DworkMNS06} constraint ensures that the distribution on messages sent by one participant does not change dramatically if that participant's data point is changed. Thus, the central server or an outside observer cannot learn much about any particular data point, guaranteeing a strong form of privacy protection (as long as the privacy parameter is small enough). Nevertheless, with enough participants, the combination of all private messages can reveal enough statistical information in aggregate in order to solve a statistical task, such as learning. LDP is the model of choice of many industrial deployments of differential privacy~\cite{ErlingssonPK14,Thakurta+17,AppleDP,MicrosoftDP}. Here we focus on \emph{non-interactive} LDP protocols, i.e., protocols in which each participant simultaneously sends a single message to the server. Non-interactive protocols are much easier to implement than multi-round interactive protocols,
particularly considering the large number
of data points which are typically necessary for LDP to be useful.

Our main goal is to characterize, for any given concept class \(\concepts\), the sample complexity of learning and refutation under the constraints of non-interactive LDP. Moreover, we aim to understand how these two problems are related to each other.

In many settings, it is trivial to take an algorithm for learning and
use it to obtain an algorithm for refutation, by executing  the
learning algorithm for accuracy $\alpha/4$, and estimating the loss of
the returned hypothesis within $\alpha/4$. Surprisingly, a converse of this simple reduction was established by~\cite{KL18}, and by~\cite{learningvsrefutation}. Unfortunately, neither of these reductions applies to the setting of non-interactive LDP, since they rely on interacting with the distribution $\dista$ adaptively. 
This leaves open the question of whether or not 
learning and refutation in the non-interactive LDP setting are equivalent tasks with respect to sample complexity.

We note that, by the equivalence proved in~\cite{klnrs} between LDP and the statistical queries (SQ) model of~\cite{Kearns93}, this also means that the relationship between the query complexity of non-adaptive SQ learning versus refutation is open.  Similarly, all our results extend to the non-adaptive SQ model.
Adaptive SQ learning has been characterized by~\cite{gencharsq}, and
this in turn implies the same characterization for sequential LDP (LDP
protocols in which each participant sends one message, which can
depend previously sent messages).  

An overview of our main results follows.
The derivation of our results 
will be presented in Section~\ref{sec:agnostic}
for the agnostic setting
and
in Section~\ref{sec:realizable} for the realizable case, after necessary
preliminaries are covered in Section~\ref{sec:prelims}.

\subsection{Characterization of agnostic learning}

Our first theorem shows that non-interactive LDP learning and refutation are equivalent (up to a logarithmic approximation) in the agnostic setting. We do so by the following theorem, which gives a characterization of the sample complexity of both problems in terms of the approximate \(\fnorminf\) norm of a natural matrix associated with the concept class \(\concepts \subseteq \{\pm 1\}^\iuni\). 


\begin{theorem}\label{thm:agnostic-main}
Let \(\concepts \subseteq \{\pm 1\}^\iuni\) be a finite concept class with concept matrix \(\qmat \in \{\pm 1\}^{\concepts \times \iuni}\), as given by Definition~\ref{def:cmatdef}. Let \(\priv > 0\), \(\acc, \conf \in (0,1/2]\). Then, to either  $(2\acc,\conf)$-learn $\concepts$ agnostically,
or $(2\acc,\conf)$-refute $\concepts$ agnostically under non-interactive \(\priv\)-LDP,
it suffices to have a sample of size
\[
    \dsize
        = O\left( \frac{\fnorminf(\qmat,\acc)^2 \cdot \log (|\concepts| / \conf)}{\priv^2 \acc^2} \right).
\]
Conversely, for some \(\acc' = \Omega\left( \frac{\acc}{\log(1/\acc)}
\right)\), and for every $\conf \le \frac12 - \Omega(1)$, the number of samples required
    to either $(\acc',\conf)$-learn agnostically or $(\acc',\conf)$-refute  $\concepts$ agnostically under non-interactive \(\priv\)-LDP
    is at least
    \[
        \dsize = \Omega \left( \frac{(\fnorminf(W,\acc )-1)^2}{\priv^2 \acc^2} \right).
    \]
\end{theorem}
In Theorem~\ref{thm:agnostic-main}, we denote by \(\fnorminf(\qmat,\acc)\) the approximate \(\fnorminf\) norm of the matrix \(\qmat\), i.e., the minimum \(\fnorminf\) norm of a matrix that approximates \(\qmat\)  up to an additive \(\acc\) entrywise. (For a definition of the \(\fnorminf\) norm, see Section~\ref{sec:prelims}.) The theorem shows that the sample complexity of both learning and refutation under non-interactive LDP can be characterized in terms of \(\fnorminf(\qmat,\acc)\). Moreover, the sample complexities of both problems are equal, up to a factor \(O(\log(1/\acc))\) loss in the accuracy parameter, and a factor \(O(\log |\concepts|)\) loss in the sample complexity. 

The main new result in Theorem~\ref{thm:agnostic-main} is the lower bound on the sample complexity of learning. 
The upper bound for both learning and refutation, as well as the lower
bound for refutation were previously shown
in~\cite{factorization}. As in the previous
proofs, we prove our lower bound via the (dual formulation) of the approximate 
 \(\fnorminf\) norm. 
 In order to give a family of distributions that is hard against learning algorithms, we define a new {\it difference} matrix,
  \(\difmat\), associated with the concept class \(\concepts\), which is more suitable for the learning lower bound. 
 Then we show that \(\fnorminf(\difmat,\acc)\) and
 \(\fnorminf(\qmat,\acc)\) are approximately equal.

 It is worth noting that, by the results of~\cite{factorization}, the
 sample complexity of estimating the loss of every concept in
 $\concepts$ up to an error $O(\acc)$ is also approximately captured by the $\fnorminf(\qmat,\acc)$. Thus, together with our new lower bound, this gives an
 approximate equivalence of this estimation problem with learning and refutation. 

\subsection{Characterization of realizable refutation}

The results above do not apply to the realizable setting, in which the underlying distribution
$\dista$ on $\iuni \times \oo$
is guaranteed to be labelled
by a concept $\concept \in \concepts$, i.e., $\min_{\concept \in
  \concepts} \loss_\dista(\concept) =
0$. In particular, the lower bounds we prove in terms of the
approximate \(\fnorminf\) norm utilize distributions that may not be
realizable.  {This is not an accident, since some concept
  classes are much easier to learn under realizable distributions. For
  example, the class of conjunctions over \(\iuni = \{0,1\}^d\) can be
  learned with polynomial in \(d\) query complexity using a
  non-adaptive SQ algorithm~\cite{Kearns93}, and, therefore, also
  with polynimial sample complexity by a non-interactive LDP
  algorithm. The \(\fnorminf\) norm of the matrix associated with this
  class is, however, exponential in \(d\), as shown
  in~\cite{factorization}. Therefore, conjunctions require exponential
  sample complexity to learn agnostically under non-interactive
  LDP. A similar result was also proved by Feldman using a reduction
  from learning parities~\cite{Feldman09}. }

Daniely and Feldman~\cite{DanielyFeldman18} showed that (for \(\concepts\) closed under negation) the sample complexity of realizable learning under non-interactive LDP is bounded from below by the margin complexity of \(\concepts\). They left open the question whether one can prove a matching upper bound. This question was resolved in the negative by \cite{feldman-margin}. The problem of characterizing the sample complexity of realizable learning under non-interactive LDP thus remains open.

While we are also not able to characterize realizable learning, we give a characterization of a realizable analog of the refutation problem, and show that realizable learning is no harder than realizable refutation. In our formulation of realizable refutation with accuracy \(\alpha\), we are given samples from some distribution \(\dista\) over \(\iuni \times \oo\), and the goal is to distinguish the cases:
\begin{itemize}
    \item \(\min_{\concept \in \concepts}\loss_{\dista}(\concept) = 0\), i.e, some concept in \(\concepts\) exactly gives the labels under \(\dista\);
        
    \item for all \(h \in \oo^\iuni\),
    \(
    \loss_\dista(h) \ge \alpha.
    \)
\end{itemize}
The \(\alpha=\frac12\) case is equivalent to the definition of refutation introduced by~\cite{learningvsrefutation}.

In this work,
we give a non-interactive LDP protocol
which may be applied towards both
realizable learning and realizable refutation.
This gives a sample complexity upper bound
for these problems
in terms of a new efficiently computable quantity \(\mcnew(\concepts,\acc)\) that we define. 
Further, we derive a lower bound
for realizable refutation in terms of 
\(\mcnew(\concepts,\acc)\), showing that our protocol
is nearly optimal for realizable refutation, and that the sample complexity of realizable refutation is an upper bound on the sample complexity of realizable learning under non-interactive LDP. Our main theorem for realizable learning is stated next. 

\begin{theorem}\label{thm:realizable-main}
Let \(\concepts \subseteq \{\pm 1\}^\iuni\) be a finite concept
class. Let \(\priv > 0\), \(\acc, \conf \in (0,1/2)\), where $\conf \le
\frac12 - \Omega(1)$. Then, to either  $(2\acc,\conf)$-learn $\concepts$ realizably,
or $(2\acc,\conf)$-refute $\concepts$ realizably under non-interactive \(\priv\)-LDP,
it suffices to have a sample of size
\[
    \dsize
        = O\left( \frac{\mcnew(\concepts,\acc)^2 \cdot \log (|\concepts| / \conf)}{\priv^2 \acc^2} \right).
\]
Conversely, for some \(\acc' = \Omega\left( \frac{\acc}{\log(1/\acc)} \right)\), the number of samples required
to $(\acc',\conf)$-refute $\concepts$ realizably under non-interactive \(\priv\)-LDP
is at least
\[
  \dsize = \Omega \left( \frac{\mcnew(\concepts,\acc)^2}{\priv^2 \acc^2} \right).
\]
\end{theorem}
See Section~\ref{sec:realizable} for the definition of
\(\mcnew(\concepts,\acc)\). On a high level, this quantity comes from
estimating a surrogate loss function for each concept in \(\concepts\)
using the factorization mechanism of \cite{factorization}. This
surrogate loss has the property that the loss of any distribution
labeled by the concept is close to \(0\), and the loss of any
distribution far from being labeled by the concept is large.


\section{Preliminaries} \label{sec:prelims}

In this section we introduce our notation
and review standard definitions
pertaining to privacy and learning.



\subsection{Norms}\label{sec:norms}

\tnote{Put 2.1 in Appendix.}
For a set $\mathcal{S}$, the $\ell_1$, $\ell_2$ and $\ell_\infty$ norms
on $\R^\mathcal{S}$ are given respectively by
\[
    \| a \|_1 = \sum_{v \in \mathcal{S}} |a_v|,
    \quad
    \| a \|_2 = \sqrt{ \sum_{v \in \mathcal{S}} (a_v)^2 },
    \quad
    \| a \|_\infty = \max_{v \in \mathcal{S}} |a_v|.
\]

Given a probability distribution $\pi$ on $\mathcal{S}$,
we consider the norms $\|\cdot\|_{L_1(\pi)}$ and $\|\cdot\|_{L_2(\pi)}$
on $\R^\mathcal{S}$, given by
\[
    \| a \|_{L_1(\pi)}
    = \sum_{v \in \mathcal{S}} \pi(v) |a_v|,
    \quad 
    \| a \|_{L_2(\pi)}
    = \sqrt{ \sum_{v \in \mathcal{S}} \pi(v) (a_v)^2}.
\]

We also take advantage of a number of matrix norms.
For norms $\|\cdot\|_\zeta$ and $\|\cdot\|_\xi$
on $\R^{\mathcal{S}'}$ and $\R^{\mathcal{S}}$ respectively,
we consider the \emph{matrix operator norm} of
$M \in \R^{\mathcal{S} \times \mathcal{S}'}$
given by
\[
    \| M \|_{\zeta \to \xi} = \max_{\ielem \in \R^\mathcal{S} \setminus \{0\}} \frac{\| M\ielem \|_\xi}{\|\ielem\|_\zeta}.
\]
For the special case of $\|M\|_{\ell_s \to \ell_t}$,
we will simply write $\|M\|_{s \to t}$.
Of particular importance are
$\|M\|_{1 \to \infty}$ which corresponds to the largest entry of $M$,
$\|M\|_{1 \rightarrow 2}$,
which corresponds to the maximum $\ell_2$ norm of a column of $M$,
and
$\|M\|_{2 \rightarrow \infty}$,
which corresponds to the maximum $\ell_2$ norm of a row of $M$.
\anote{Hyphenate norms?}

The \emph{inner product} of two matrices $M$ and $N$ in
$\R^{\mathcal{S} \times \mathcal{S}'}$ is defined by
$M\bullet N = \tr(M^\top N) = \sum_{u \in \mathcal{S}, v\in
  \mathcal{S}'} m_{u,v} n_{u,v}$.

The \emph{factorization norm} known as
the $\fnorminf$ norm is given for
$M \in \R^{\mathcal{S} \times \mathcal{S}'}$ by
\[
    \fnorminf(M) = \min\{\|R\|_{2 \to \infty}\|A\|_{1 \to 2} : RA = M\}.
  \]
The \(\fnorminf\) norm is, indeed, a norm, i.e., it is non-negative,
\(\fnorminf(M)=0\) if and only if \(M = 0\), for any real \(s\) we
have \(\fnorminf(sM) = |s|\fnorminf(M)\), and we also have the
triangle inequality \(\fnorminf(M+N) \le \fnorminf(M) + \fnorminf(N)\).

The approximate \(\fnorminf\) norm is the smallest \(\fnorminf\) norm of a matrix that approximates the given matrix entrywise up to an additive \(\acc\), i.e., 
\[
\fnorminf(M,\acc)
= 
\min\{\fnorminf(\widetilde{M}): \|\widetilde{M} - M\|_{1\to\infty} \le \acc\}.
\]
The dual \(\fnorminf\) norm of a matrix $G$ in $\R^{\mathcal{S} \times \mathcal{S}'}$ is given by 
\[
\fnorminf^*(N) = \max\{M\bullet N: \fnorminf(M) \le 1\}
=
\max_{f,g} \sum_{u \in \mathcal{S},v\in \mathcal{S}'} n_{u,v}f(u)g(v),
\]
where the second max ranges over functions \(f:\mathcal{S}\to B_2\) and \(g:\mathcal{S}'\to B_2\) that map the index sets of the rows and columns of \(N\), respectively, to vectors of \(\ell_2\) norm at most 1.

\subsection{Differential privacy}

Let $\uni$ denote the \emph{data universe}.
A generic element from $\uni$ will be denoted by $\elem$.
We consider \emph{datasets} of the form
$\ds = (\elem_1,\dots,\elem_\dsize) \in \uni^n$,
each of which is identified with its \emph{histogram}
$\hist \in \Z_{\ge 0}^{\uni}$ where, for every $\elem \in \uni$, $\hist_{\elem} = | \set{i : \dsrow_i = \elem} |$,
so that $\| \hist \|_1 = \dsize$.
To refer to a dataset, we use $\ds$ and $h$ interchangeably.
A pair of datasets $\ds = (\elem_1,\dots,\elem_i,\dots,\elem_\dsize)$
and $\ds' = (\elem_1,\dots,\elem_i',\dots,\elem_\dsize)$ are called \emph{adjacent} if $\ds'$ is obtained from $\ds$ by replacing
an element $\elem_i$ of $\ds$ with a new universe element $\elem_i'$.

For a parameter $\priv>0$,
an \emph{$\priv$-differentially private} ($\priv$-DP) mechanism~\cite{DworkMNS06}
is a randomized function $\mech:\uni^\dsize \to \outspace$
which, for all adjacent datasets $\ds$ and $\ds'$,
for all outcomes $S \subseteq \outspace$, satisfies
\[
    \Pr_\mech[\mech(\ds) \in S] \le e^\priv\Pr_\mech[\mech(\ds') \in S].
\]

Of special interest are $\priv$-differentially private mechanisms $\mech_i:\uni \to \loutspace$
which take a singleton dataset $\ds = \{\elem\}$ as input.
These are referred to as \emph{local randomizers}.
A sequence of $\priv$-differentially private local randomizers $\mech_1,\dots,\mech_\dsize$,
together with a \emph{post-processing function}
$\post:\loutspace^\dsize \to \outspace$,
specify a \emph{(non-interactive) locally
$\priv$-differentially private} ($\priv$-LDP) mechanism
$\mech:\uni^\dsize \to \outspace$~\cite{EvfimievskiGS03,DworkMNS06, KasiviswanathanLNRS08}.
When the local mechanism $\mech$ is applied to a dataset $\ds$,
we refer to
    $\trans(\ds) = (\mech_1(\elem_1),\dots,\mech_\dsize(\elem_\dsize))$
as the \emph{transcript} of the mechanism.
Then the output of the mechanism is given by
\(
    \mech(\ds) = \post(\trans(\ds)).
\)

%
A \emph{linear query} is specified by a bounded function
$q: \uni \to \R$.
Abusing notation slightly, its answer on a dataset $\ds$ is given by
$\query(\ds) = \frac{1}{\dsize} \sum_{i=1}^{\dsize} \query(\dsrow_i)$.
We also extend this notation to distributions: if $\dist$ is a
distribution on $\uni$, then we write $\query(\dist)$ for
 $\ex{\elem \sim \dist}{\query(\elem)}$.
A \emph{workload} is a set of linear queries
$\queries = \set{\query_1,\dots,\query_\qsize}$,
and
$\queries(\ds) = (\query_1(\ds),\dots,\query_\qsize(\ds))$
is used to denote their answers. The answers on a distribution $\dist$
on $\uni$ are denoted by 
$\queries(\dist) = (\query_1(\dist),\dots,\query_\qsize(\dist))$.
We will often represent $\queries$ by its \emph{workload matrix} $W \in \R^{\queries \times \uni}$ with entries $w_{\query,\elem} = \query(\elem)$.
In this notation, the answers to the queries are given by
$\frac{1}{n} \qmat \hist$, where we recall that $\hist$ is the
histogram of the dataset $\ds$.
We will often use $\queries$ and $\qmat$ interchangeably.

\subsection{PAC learning}

A concept $\concept:\iuni \rightarrow \oo$
from a concept class $\concepts$
assigns to each sample $\ielem \in \iuni$ a label $\concept(\ielem)$.
The \emph{empirical loss}
of the concept $\concept$ on a dataset
$
    \ds
    = ((\ielem_1,\lab_1),\dots,(\ielem_\dsize,\lab_\dsize))
    \in \left(\iuni \times \oo \right)^\dsize
$,
denoted $\loss_{\ds}(\concept)$,
is given by
\[
    \loss_{\ds}(\concept)
    = \frac{1}{\dsize} \sum_{i = 1}^{\dsize} (\I[\concept(\ielem_i) \ne \lab_i])
\]

For a distribution $\dist$ on $\iuni \times \oo$,
the \emph{population loss}
of $\concept$ on $\dist$,
denoted $\loss_{\dist}(\concept)$
is given by
\[
\loss_{\dist}(\concept)=
\ex{(\ielem,\lab) \sim \dist}{\I[\concept(\ielem) \ne \lab}
    = \pr{(\ielem,\lab) \sim \dist}{\concept(\ielem) \neq \lab}.
\]

We will say that a mechanism
$\mech:(\iuni \times \oo)^\dsize \rightarrow \oo^\iuni$
\emph{($\acc$,$\conf$)-learns $\concepts$ agnostically}
with $\dsize$ samples if,
for any distribution $\dist$ over $\iuni \times \oo$,
given as input a random dataset $\ds$
drawn i.i.d. from $\dist$,
the mechanism returns some
\emph{hypothesis} $\hyp \in \iuni \rightarrow \oo$
which satisfies
\begin{equation}\label{eq:aglearning}
    \pr{\ds,\mech}{
        \loss_\dist(\hyp)
        \le
        \min_{\concept \in \concepts} \loss_\dist(\concept)
        + \alpha
    } \ge 1 - \conf.
\end{equation}
Realizable learning is an important special case
of agnostic learning where the underlying distribution
agrees with some concept.
We say that
$\mech:(\iuni \times \oo)^\dsize \rightarrow \oo^\iuni$
\emph{($\acc$,$\conf$)-learns $\concepts$ realizably}
with $\dsize$ samples if,
whenever $\dist$ is a distribution over $\iuni \times \oo$
which satisfies
\(
\loss_\dist(\concept) = 0
\)
for some unknown $\concept \in \concepts$,
then,
given a random dataset $\ds$
drawn i.i.d. from $\dist$,
the mechanism returns
a \emph{hypothesis} $\hyp \in \iuni \rightarrow \oo$
which satisfies
\begin{equation}\label{eq:real-learning}
    \pr{\ds,\mech}{
        \loss_\dist(\hyp)
        \le
        \alpha
    } \ge 1 - \conf.
\end{equation}

The problem of refutation
asks whether the underlying distribution
is well approximated by the concept class.
In particular,
for $\theta \in [0,1]$,
we will say that
$\mech:(\iuni \times \oo)^\dsize \rightarrow \oo$
\emph{$(\acc,\conf)$-refutes $\concepts$ 
for threshold $\theta$}
if the following two conditions are met:
\begin{enumerate}
    \item
        When $\dist$ is a distribution
        on $\iuni \times \oo$ which satisfies
        $\loss_\dist(\concept) \le \theta$
        for some $\concept \in \concepts$,
        \[
            \pr{\ds,\mech}{\mech(\ds) = 1}
            \ge 1 - \beta;
        \]
    \item
        When $\dist$ is a distribution
        on $\iuni \times \oo$ which,
        for all $\hyp \in \oo^\iuni$,
        satisfies
        $\loss_\dist(\hyp) \ge \theta + \alpha$,
        then
        \[
            \pr{\ds,\mech}{\mech(\ds) = -1}
            \ge 1 - \beta.
        \]
\end{enumerate}
Realizable refutation is a special case
of agnostic refutation where
the goal is to recognize whether the underlying distribution
is labeled by a concept from the concept class.
We say that
$\mech:(\iuni \times \oo)^\dsize \rightarrow \oo$
\emph{$(\acc,\conf)$-refutes $\concepts$ realizably} if it
\((\acc,\conf)\)-refutes \(\concepts\) for threshold \(0\). This
definition agrees with the definition of \cite{learningvsrefutation} when
\(\acc = \frac12\), with the minor difference that Vadhan's
definition is stated in terms of datasets and empirical loss, rather
than data distributions and population loss. We note that
$\loss_\dist(\hyp) \ge \frac12$ for all $\hyp \in \oo^\iuni$ if and
only if each data point has an independent uniformly random label,
which is how the second condition of the refutation problem is stated
by \cite{learningvsrefutation}.


\section{Refutation versus Learning: Agnostic Case}
\label{sec:agnostic}


As mentioned in the introduction, \cite{factorization} (Theorems 24 and 25) gave sample complexity upper bounds for both agnostic learning and refutation for non-interactive LDP in terms of the approximate $\fnorminf$, as well as a nearly tight lower bound for agnostic refutation. However, it left open the question of lower bounds for agnostic learning under non-interactive LDP.
Our main theorem, stated next, resolves this by giving a nearly tight lower bound in terms of the approximate $\fnorminf$ norm of a
natural matrix associated with $\concepts$.
Theorem \ref{thm:agnostic-main} thus follows by combining Theorems 24 and 25 from \cite{factorization} together with Theorem \ref{thm:aglearnlbqmat}.\footnote{While Theorems 24 and 25 from \cite{factorization} are stated in terms of agnostic learning, their definition of agnostic learning is non-standard and requires the learner to output a hypothesis as well as the loss it achieves. Thus the upper bounds hold for the standard definition of agnostic learning, while the lower bound only holds for refutation.} 

\begin{definition}\label{def:cmatdef}
    Let $\concepts \subseteq \oo^\iuni$
    be a concept class.
    The concept matrix
    $\cmat \in \oo^{\concepts \times \iuni}$
    of $\concepts$
    is the matrix with entries given by \( \cent_{\concept,\ielem} = \concept(\ielem).\)
    \cut{\begin{equation}\label{eq:cmatdef}
        w_{\concept,\ielem} = \concept(\ielem).
    \end{equation}}
\end{definition}

\begin{theorem}\label{thm:aglearnlbqmat}

    Let $\concepts \subseteq \oo^\iuni$ be a concept class
    with concept matrix $\qmat \in \oo^{\concepts \times \iuni}$
    as given by Definition~\ref{def:cmatdef}.
    Let $\priv > 0$, $\acc,\conf \in (0,1/2)$, where
    $\conf \le \frac12 - \Omega(1)$.
    Then, for some
    $\acc'
    = \Omega \left( \frac{\acc}{\log(1/\acc)} \right)$,
    \cut{and any sufficiently large constant $C>0$,
    if
    $
        \frac{\fnorminf(W,\acc)}{\priv^2 \acc^2}
        \ge \frac{C \log 2 |\concepts|}{(\acc')^2} 
    $,
    then} under non-interactive $\priv$-LDP,
    the number of samples required
    to $(\acc',\conf)$-learn $\concepts$ agnostically
    is at least
    \[
        \dsize = \Omega \left( \frac{(\fnorminf(W,\acc) -1)^2}{\priv^2 \acc^2} \right).
    \]

\end{theorem}

\subsection{Difference matrix}

Theorem~\ref{thm:aglearnlbqmat}
is given in terms of the concept matrix associated with the concept class;
however, our proof of this result
will focus instead on the {\it difference} matrix associated with
the concept class, defined below. 

\begin{definition}\label{def:difmatdef}

    The\cut{ (unsymmetrized)} difference matrix
    of a concept class
    $\concepts:\iuni \rightarrow \oo$
    is the matrix
    $\difmat \in \oo^{\concepts^2 \times \iuni}$
    with entries given, for
    $\concept,\conceptb \in \concepts,
    \ielem \in \iuni$, by
    \begin{equation} \label{eq:difmatdef}
        \difent_{(\concept,\conceptb), \ielem}
        =
        \frac{1}{2} \left( \concept(\ielem) - \conceptb(\ielem) \right)
        = 
        \begin{cases}
            0  & \text{if } \concept(\ielem) = \conceptb(\ielem) \\
            -1 & \text{if } \concept(\ielem) = -1, \conceptb(\ielem) = +1 \\
            +1 & \text{if } \concept(\ielem) = +1, \conceptb(\ielem) = -1.
        \end{cases}
    \end{equation}
    \cut{The symmetrized difference matrix
    of $\concepts$ is the matrix
    $\symdifmat
    \in \oo^{\concepts^2 \times (\iuni \times \oo)}$
    with entries given, for
    $\concept,\conceptb \in \concepts,
    \ielem \in \iuni, \lab \in \oo$, by
       $ \symdifent_{(\concept,\conceptb), (\ielem,\lab)}
        = \lab \cdot
            \difent_{(\concept,\conceptb), \ielem}.$
    }
\end{definition}
The difference matrix is one of the key ideas that enables the proof of Theorem~\ref{thm:aglearnlbqmat}. We will use a dual formulation of \(\fnorminf(\difmat,\acc)\) to construct pairs of hard distributions for our lower bound, each pair corresponding to a pair of concepts \(\concept,\conceptb\in \concepts\). The structure of the difference matrix will help us ensure that no correct agnostic learning algorithm can output, with  high probability, the same hypothesis for both distributions in a pair. It is not apparent how to guarantee this property when working directly with the concept matrix \(\qmat\).
\cut{To the motivate the definition
of the difference matrix,
we note that, for a concept class
$\concepts \subseteq \oo^\iuni$,
by answering the workload of queries
$
    \{\query_{\concept,\conceptb}\}_{\concept,\conceptb \in \concepts}
$
corresponding to 
the symmetrized difference matrix
$\symdifmat$ of $\concepts$,
we obtain, for each $\concept,\conceptb \in \concepts$,
an estimate of
\begin{equation}\label{eq:queryvalues}
    \query_{\concept,\conceptb}(\dista)
    = \ex{(\ielem,\lab) \sim \dista}{\lab \cdot d_{(\concept,\conceptb),\ielem}}
    = \ex{(\ielem,\lab) \sim \dista}{\frac{\lab}{2} \cdot \left(\concept(\ielem) - \conceptb(\ielem)\right)}
    = \loss_{\dista}(\concept) - \loss_{\dista}(\conceptb).
\end{equation}
In other words,
answering these queries gives the difference in loss
for each pair of concepts.}
Nevertheless, the following lemma shows that $\fnorminf(\difmat,\alpha)$ and $\fnorminf(\qmat,\acc)$ are essentially the same.
See Appendix~\ref{ap:difmatqmat} for the proof.

\begin{lemma}\label{lm:normrelation-main}
    Let $\concepts$ be a concept class
    with concept matrix
    $\qmat \in \R^{\concepts \times \iuni}$
    and difference matrix
    $\difmat \in \R^{\concepts^2 \times \iuni}$.
    Then
    $\fnorminf(\difmat,\acc)
    \leq \fnorminf(\qmat,\acc)$. Conversely, 
    $\fnorminf(\qmat,\acc) \leq 2 \fnorminf(\difmat,\acc/2) +1 $, and if $\concepts$ is closed under negation then $\fnorminf(\qmat,\acc) \leq \fnorminf(\difmat,\acc)$.
\end{lemma}

The next lemma is the same as Theorem \ref{thm:aglearnlbqmat}, but with $W$ replaced by the difference matrix $D$. 
Theorem~\ref{thm:aglearnlbqmat}
is an immediate consequence of Lemma~\ref{lm:aglearnlb},
together with 
Lemma~\ref{lm:normrelation-main}.

\begin{lemma}\label{lm:aglearnlb}

    Let $\concepts \subseteq \oo^\iuni$ be a concept class
    with concept matrix $\difmat \in \oo^{\concepts \times \iuni}$
    as given by Definition~\ref{def:cmatdef}. Let $\priv,\acc, \acc',
    \conf$ be as in Theorem~\ref{thm:aglearnlbqmat}. Then, under non-interactive $\priv$-LDP,
    the number of samples required
    to $(\acc',\conf)$-learn $\concepts$ agnostically
    is at least
    \[
        \dsize = \Omega \left( \frac{\fnorminf(\difmat,\acc)^2}{\priv^2 \acc^2} \right).
    \]
\end{lemma}

The rest of this section is devoted to the proof of  Lemma~\ref{lm:aglearnlb}.


\subsection{KL-divergence bound}
\label{sec:kl}


For a mechanism
$\mech:\uni^\dsize \rightarrow \outspace$
in the local model,
and a probability distribution $\nu$ on
$\uni^\dsize$,
we use $\trans(\nu)$
to denote the distribution
of the mechanism's transcript
when its input is sampled from $\nu$. The following information
theoretic lemma will be used to obtain our lower bounds.

\begin{lemma}\label{lm:kl-div} \cite{factorization}
Let $\priv \in (0,1]$, and let
$\mech:\uni \rightarrow \outspace$
be a non-interactive $\priv$-LDP protocol.
Then, for families $\{\dista_1,\dots,\dista_\qsize\}$ and
$\{\distb_1,\dots,\distb_\qsize\}$ of distributions on $\uni$,
together with a distribution $\distv$ over $[\qsize]$,
\begin{align*}
     & \ex{V \sim \distv}{\div(\trans(\dista_V^\dsize)\|\trans(\distb_V^\dsize))} 
    \le O(\dsize \priv^2)\cdot 
    \max_{f \in \R^\uni: \|f\|_\infty \le 1} \ex{V \sim \distv}{\left(\ex{\elem \sim \dista_V}{f_\elem} - \ex{\elem \sim \distb_V}{f_\elem}\right)^2}.
\end{align*}
In matrix notation, define
the matrix $M \in \R^{[\qsize] \times \uni}$ by $m_{v, \elem} = (\dista_v(\elem) - \distb_v(\elem))$.
Then
\[
    \ex{V \sim \distv}{\div(\trans(\dista_V^\dsize)\|\trans(\distb_V^\dsize))}
    \le O(\dsize \priv^2)\cdot \|M\|_{\ell_\infty \to L_2(\distv)}^2.
\]
\end{lemma}

Note that the statement of Lemma~\ref{lm:kl-div}
is slightly different from
 the statement given in \cite{factorization}, but the result as stated here
is an immediate consequence of the original proof.

Our lower bound against agnostic learning will construct
families $\{\dista_1,\dots,\dista_\qsize\}$ and
$\{\distb_1,\dots,\distb_\qsize\}$ of distributions on
$\iuni \times \{\pm 1\}$,
as well as a distribution $\distv$ over $[\qsize]$.
The idea will be to construct these distributions
so that, when $\mech$ is an agnostic learner
for $\concepts$, then,
for any fixed $v \in [k]$,
\[
    \div(\trans(\dista_v^\dsize)\|\trans(\distb_v^\dsize))
     \ge 2\tv(\trans(\dista_v^\dsize)\|\trans(\distb_v^\dsize))^2
    \ge \Omega(1),
\]
where the first inequality is just Pinsker's inequality, and the second one will follow from our construction of $\dista_v$ and $\distb_v$.

\subsection{Duality and Hard Distributions}
\label{sec:harddistros}

For the construction of hard families of distributions,
it will be convenient to make use
of the following dual formulation,
shown in \cite{factorization}.

\begin{lemma}
    Let $\difmat \in \R^{\concepts^2 \times \iuni}$
    be the difference matrix of a concept class $\concepts$,
    as given by \eqref{eq:difmatdef}.
    Then,
    \begin{equation}\label{eq:agnosticdual}
        \fnorminf(\difmat,\acc)
        = \max
            \left\{
                \frac{\difmat \bullet \dumat - \acc \|\dumat\|_1}{\fnorminf^*(\dumat)} 
                \ : \
                \dumat \in \concepts^2 \times \iuni, \
                \dumat \ne 0
            \right\}.
    \end{equation}
\end{lemma}

For an arbitrary concept class
$\concepts \subseteq \oo^\iuni$,
let $\dumat \in \R^{\concepts^2 \times \iuni}$ witness \eqref{eq:agnosticdual},
so that
\begin{equation}\label{eq:dualinst}
    \fnorminf(\difmat, \acc)
    = \frac{\difmat \bullet \dumat - \acc \|\dumat\|_1}{\fnorminf^*(\dumat)}.
\end{equation}
By normalizing $\dumat$, we may assume,
without loss of generality, that $\|\dumat\|_1 = 1$. Moreover, we can assume that, for any \(\concept,\conceptb\in \concepts\),
    $\sum_{\ielem \in \iuni}\difent_{(\concept,\conceptb),\ielem} {u}_{(\concept,\conceptb),\ielem} \ge 0.$
Otherwise, \(\dumat\) cannot achieve \eqref{eq:agnosticdual}, since we can multiply the row of \(\dumat\) indexed by \((\concept,\conceptb)\) by \(-1\), which increases \(\difmat \bullet \dumat\) and does not change \(\|\dumat\|_1\) or \(\fnorminf^*(\dumat)\).

We will consider the matrices
${U}^+, {\dumat}^-
\in \R^{\concepts^2 \times \iuni}$
with non-negative entries which satisfy
${\dumat} = {\dumat}^+ - {\dumat}^-$,
so that ${\dumat}^+$ and ${\dumat}^-$ correspond to the positive and negative
entries of $\dumat$ respectively.
We define the distribution $\distv$ on
$\concepts^2$ by
\begin{equation}\label{eq:distdefv}
    \distv(\concept,\conceptb)
    = \sum_{\ielem \in \iuni } {u}_{(\concept,\conceptb),\ielem}.
\end{equation}
Then, for $\concept,\conceptb \in \concepts$, consider
the distribution $\dista_{\concept,\conceptb}$ on $\iuni \times \oo$
given by
\begin{equation}\label{eq:distdefa}
    \dista_{\concept,\conceptb}(\ielem,1)
    = \frac{{u}^+_{(\concept,\conceptb),\ielem}}{\distv(\concept,\conceptb)},
    \qquad\qquad
    \dista_{\concept,\conceptb}(\ielem,-1)
    = \frac{{u}^-_{(\concept,\conceptb),\ielem}}{\distv(\concept,\conceptb)}
\end{equation}
Similarly, let $\distb_{\concept,\conceptb}$ be the distribution
on $\iuni \times \{\pm 1\}$ given by
\begin{equation}\label{eq:distdefb}
    \distb_{\concept,\conceptb}(\ielem,1)
    = \frac{{u}^-_{(\concept,\conceptb),\ielem}}{\distv(\concept,\conceptb)},
    \qquad\qquad
    \distb_{\concept,\conceptb}(\ielem,-1)
    = \frac{{u}^+_{(\concept,\conceptb),\ielem}}{\distv(\concept,\conceptb)}.
\end{equation}
Since ${\dumat}$ has unit $\ell_1$ norm,
the above distributions are well-defined.
Note that $\dista_{\concept,\conceptb}$ and $\distb_{\concept,\conceptb}$
have the same marginal on $\iuni$
which we denote $\kappa_{(\concept,\conceptb)}$.
In particular,
\(
    \kappa_{(\concept,\conceptb)}(\ielem)
    = \frac{|u_{(\concept,\conceptb),\ielem}|}{\distv(\concept,\conceptb)}.
\)
Meanwhile, $\dista_{\concept,\conceptb}$ always gives $\ielem$
the label $\lab = \sign(u_{(\concept,\conceptb),\ielem})$,
while $\distb_{\concept,\conceptb}$ always gives $\ielem$
the label $\lab = -\sign(u_{(\concept,\conceptb),\ielem})$.
It will be useful to have notation
for one of these labelling functions,
so define
$s_{\concept,\conceptb}:\iuni \rightarrow \oo$ by
    $s_{\concept,\conceptb}(\ielem)
    = \sign(u_{(\concept,\conceptb),\ielem}).$

Consider the following relationship between
$\dumat$ and the distributions we have constructed.
\begin{align*}
    u_{(\concept,\conceptb),\ielem}
    &= \distv(\concept,\conceptb) \left(
        \dista_{\concept,\conceptb}(\ielem,1) - \distb_{\concept,\conceptb}(\ielem,1)
    \right) 
    = \distv(\concept,\conceptb) \kappa_{\concept,\conceptb}(\ielem) s_{\concept,\conceptb}(\ielem).
\end{align*}
Note that
\begin{align}
    \sum_{\ielem \in \iuni} d_{(\concept,\conceptb),\ielem} u_{(\concept,\conceptb),\ielem}
    \cut{&= \sum_{\ielem \in \iuni} \frac{1}{2} \cdot
        \left( \concept(\ielem) - \conceptb(\ielem) \right) \cdot \distv(\concept,\conceptb) \cdot
        \kappa_{\concept,\conceptb}(\ielem) \cdot s_{\concept,\conceptb}(\ielem) \\}
    &=  \distv(\concept,\conceptb)
        \cdot \sum_{\ielem \in \iuni} 
        \frac{1}{2} \cdot\kappa_{\concept,\conceptb}(\ielem) \cdot
        \left[ \concept(\ielem)s_{\concept,\conceptb}(\ielem) - \conceptb(\ielem)s_{\concept,\conceptb}(\ielem) \right] \\
    &= \distv(\concept,\conceptb) \cdot \left( \loss_{\dista_{\concept,\conceptb}}(\concept) - \loss_{\dista_{\concept,\conceptb}}(\conceptb) \right).
    \label{eq:dualdifrowdot}
\end{align}
Similarly,
\begin{equation}\label{eq:dualdifrowdotb}
    \sum_{\ielem \in \iuni} d_{(\concept,\conceptb),\ielem} u_{(\concept,\conceptb),\ielem}
    = \distv(\concept,\conceptb) \cdot \left( \loss_{\distb_{\concept,\conceptb}}(\conceptb) - \loss_{\distb_{\concept,\conceptb}}(\concept) \right).
\end{equation}
Hence,
\begin{multline*}
    \difmat \bullet \dumat
    \cut{= \sum_{\concept, \conceptb \in \concepts, \ielem \in \iuni} d_{(\concept,\conceptb),\ielem} u_{(\concept,\conceptb),\ielem}}
    = \ex{(\concept,\conceptb) \sim \distv}{
        \left( \loss_{\dista_{\concept,\conceptb}}(\conceptb) - \loss_{\dista_{\concept,\conceptb}}(\concept) \right)} 
        = \ex{(\concept,\conceptb) \sim \distv}{
        \left( \loss_{\distb_{\concept,\conceptb}}(\concept) - \loss_{\distb_{\concept,\conceptb}}(\conceptb) \right)}.
\end{multline*}
\cut{and similarly
\[
    \difmat \bullet \dumat
    = \ex{(\concept,\conceptb) \sim \distv}{
        \left( \loss_{\distb_{\concept,\conceptb}}(\concept) - \loss_{\distb_{\concept,\conceptb}}(\conceptb) \right)}.
\]}

Whenever \(\concepts\) contains at least two distinct concepts, $\fnorminf(\difmat,\acc)>0$, and
then \eqref{eq:dualinst} implies $\difmat \bullet \dumat > \acc$.
By the equations above, this implies that,
on average with respect to $(\concept,\conceptb) \sim \distv$,
the loss of $\concept$ is greater by $\acc$
than the loss of $\conceptb$ on $\dista_{\concept,\conceptb}$.
Likewise, on average,
the loss of $\conceptb$ is greater by $\acc$
than the loss of $\concept$ on $\distb_{\concept,\conceptb}$.
We will see later that,
if we can obtain these properties in the worst case
over all
$(\concept,\conceptb)$,
rather than only on average,
then no hypothesis can fit both
$\dista_{\concept,\conceptb}$ and $\distb_{\concept,\conceptb}$ for any \(\concept,\conceptb \in \concepts\).
The following section modifies the distributions
we have constructed so as to obtain
the required properties.

\subsection{Lower bound derivation}
\label{sec:aglb}

To make sure that the inequalities between losses from the previous
subsection hold in the  worst-case rather than on an average, we apply
a geometric binning trick, given by the next lemma.

\begin{lemma}[\cite{factorization}]\label{lm:binning}
    Suppose that $a_1, \ldots, a_\qsize \in [0,1]$
    and that $\distv$ is a probability distribution over $[\qsize]$.
    Then for any $\conf \in (0,1]$,
    there exists a set $S \subseteq [\qsize]$ such that
    $\distv(S) \cdot \min_{v\in S}a_v \ge \frac{\sum_{v=1}^\qsize\distv(v)a_v - \conf}{O(\log(1/\conf))}$.
\end{lemma}

The next lemma applies Lemma~\ref{lm:binning} to our hard
distributions, 
while also introducing some properties which will be useful later on.

\begin{lemma}\label{lm:hard-distros}

Let $\concepts$ be a concept class with difference matrix $\difmat$.
Let $\dumat \in \R^{\concepts^2 \times \iuni}$,
$\|\dumat\|_1 = 1$, satisfy \eqref{eq:dualinst}.
Then there exist probability distributions
$\widetilde{\dista}_{\concept,\conceptb}$
and $\widetilde{\distb}_{\concept,\conceptb}$
over $\iuni \times \{\pm 1\}$,
and a distribution $\widetilde{\distv}$ over $\concepts^2$
such that: 
\begin{enumerate}
   \item
        For all $(\concept,\conceptb)$ in the support of $\widetilde{\distv}$,
        $
            \loss_{\widetilde\dista_{\concept,\conceptb}}(\concept) - \loss_{\widetilde\dista_{\concept,\conceptb}}(\conceptb) 
            \ge \frac{\acc}{O(\log(1/\acc))}
        $.\label{crit:dista}
    \item
        For all $(\concept,\conceptb)$ in the support of $\widetilde{\distv}$,
        $
            \loss_{\widetilde\distb_{\concept,\conceptb}}(\conceptb) - \loss_{\widetilde\distb_{\concept,\conceptb}}(\concept) 
            \ge \frac{\acc}{O(\log(1/\acc))}
        $. \label{crit:distb}
 
    \item
        The matrix
        $
            \widetilde{\dumat} \in \R^{\concepts^2 \times \iuni}
        $
        with entries
        \[
          \widetilde{u}_{v,\ielem}
          = \widetilde{\distv}(v) \cdot ( \widetilde{\dista}_v(\ielem,1) - \widetilde{\distb}_v(\ielem,1) )
          = - \widetilde{\distv}(v) \cdot ( \widetilde{\dista}_v(\ielem,-1) - \widetilde{\distb}_v(\ielem,-1) )
        \]
        satisfies
        $
            \fnorminf^\ast(\widetilde{\dumat})
            \le \frac{\acc\fnorminf^\ast(\dumat)}{\difmat\bullet\dumat}.
        $
\end{enumerate}

\end{lemma}

\begin{proof}
Let $\distv$, together with $\dista_{\concept,\conceptb}$
and $\distb_{\concept,\conceptb}$,
be defined as in \eqref{eq:distdefv}, \eqref{eq:distdefa}
and \eqref{eq:distdefb}.
We will apply Lemma~\ref{lm:binning}
to the values given, for $\concept,\conceptb \in \concepts$, by
\[
    a_{\concept,\conceptb}
    = \loss_{\dista_{\concept,\conceptb}}(\concept) - \loss_{\dista_{\concept,\conceptb}}(\concept)
    = 
    \loss_{\distb_{\concept,\conceptb}}(\concept) - \loss_{\distb_{\concept,\conceptb}}(\concept).
\]

Recall that we may assume, that for all $\concept,\conceptb \in
\concepts$ we have
\[
\sum_{\ielem \in \iuni}\difent_{(\concept,\conceptb),\ielem} {u}_{(\concept,\conceptb),\ielem} \ge 0.
\]
Together with \eqref{eq:dualdifrowdot},
this gives
$a_{\concept,\conceptb} \ge 0$
for all
$\concept,\conceptb \in \concepts$.

By Lemma~\ref{lm:binning},
there exists some $S \subseteq \concepts^2$
such that
\begin{align*}
    \distv(S) \cdot \min_{(\concept,\conceptb) \in S} 
        \left(
            \loss_{\dista_{\concept,\conceptb}}(\concept) - \loss_{\dista_{\concept,\conceptb}}(\conceptb) 
        \right)&=     \distv(S) \cdot \min_{(\concept,\conceptb) \in S}\left(
            \loss_{\distb_{\concept,\conceptb}}(\conceptb) - \loss_{\distb_{\concept,\conceptb}}(\concept) 
        \right)\\
        &\ge \frac{\ex{(\concept,\conceptb) \sim \distv}{\loss_{\dista_{\concept,\conceptb}}(\concept) - \loss_{\dista_{\concept,\conceptb}}(\conceptb)} - \acc/4}{O(\log(1/\acc))}
        = \frac{\difmat \bullet \dumat - \acc/4}{O(\log(1/\acc))}.
\end{align*}

Let $\widetilde{\distv}$ be defined by
\[
  \widetilde{\distv}(\concept,\conceptb)
  = \begin{cases}
    \distv(\concept,\conceptb) / \distv(S), & \text{if } \concept,\conceptb \in S \\
    0, & \text{otherwise.}
  \end{cases}
\]
Let also \(\tau = \frac{\acc}{\difmat \bullet \dumat} \in (0,1)\).
For $(\concept,\conceptb) \in S$, let
$\widetilde{\dista}_{\concept,\conceptb} = \dista_{\concept,\conceptb}$
and
\[\widetilde{\distb}_{\concept,\conceptb}
= (1-\tau\distv(S)) \dista_{\concept,\conceptb} +  \tau\distv(S)\distb_{\concept,\conceptb}.\]
Then, for $(\concept,\conceptb) \in S$, it holds that
\[
    \widetilde{\dista}_{\concept,\conceptb} - \widetilde{\distb}_{\concept,\conceptb}
    = \tau \cdot \distv(S) \cdot \left( \dista_{\concept,\conceptb} - \distb_{\concept,\conceptb} \right)
\]
Hence, the matrix
$
    \widetilde{\dumat} 
$
defined in the statement of the lemma satisfies
\[
    \widetilde{u}_{(\concept,\conceptb),\ielem}
    =
    \begin{cases}
        \tau u_{(\concept,\conceptb),\ielem}, & \text{if }(\concept,\conceptb) \in S \\
        0, & \textnormal{otherwise.}
    \end{cases}
\]
It is easy to see from the definition
of $\fnorminf^*$ that this implies
$\fnorminf^*(\widetilde{\dumat}) \le \tau \fnorminf^*(\dumat) = \frac{\acc\fnorminf^*(\dumat)}{\difmat \bullet \dumat}$.
\end{proof}

We also want to bound the operator norm,
which appears in Lemma~\ref{lm:kl-div},
in terms of $\fnorminf^*(\dumat)$.
To do so, we use the following lemma from~\cite{factorization}.

\begin{lemma}[\cite{factorization}]\label{lm:gammastar-inftyto2-nonuniform}
  Let $\dumat$ and $M$ be $\qsize \times \usize$ matrices, and let $\distv$ be
  a probability distribution on $[\qsize]$ such that,
  for any $i \in [\qsize], j \in [\usize]$,
  we have $u_{i,j} = \distv(i)m_{i,j}$.
  Then there exists a probability distribution
  $\widehat{\distv}$ on $[\qsize]$,
  with support contained in the support of $\distv$,
  such that
    $
        \| M \|_{\ell_\infty \to L_2(\widehat{\distv})}
        \le 4\fnorminf^\ast(\dumat).
    $
\end{lemma}

Recall that we also want to obtain a lower bound on
$
    \tv(\trans(\dista_{\concept,\conceptb}^\dsize)\|\trans(\distb_{\concept,\conceptb}^\dsize))
$
when $\mech$ is a learning algorithm for $\concepts$.
For this purpose, we apply the following lemma.
The main observation in the proof is that, for any hypothesis \(\hyp:\concepts \to \{\pm 1\}\), and any distributions $\dista$ and $\distb$ satisfying the conditions of the lemma, we have \(\loss_\dista(\hyp) + \loss_\distb(\hyp)=1\).

\begin{lemma}\label{lm:tv}
    Let $\dista$ and $\distb$
    be distributions on $\iuni \times \{\pm 1\}$.
    Assume that $\dista$ and $\distb$
    have the same marginal on $\iuni$.
    Also assume that $\dista$ is labelled
    by some $s:\iuni \rightarrow \{\pm 1\}$
    while $\distb$ is labelled by $-s$.
    Finally, assume that for some $\concept,\conceptb \in \concepts$,
    $
        \loss_{\distb}(\conceptb)
        -
        \loss_{\distb}(\concept) > \acc.
    $
If $\hyp:\iuni \rightarrow \{\pm 1\}$
    satisfies
        $ \loss_{\dista}(\hyp)
        \le \loss_{\dista}(\conceptb) + \acc/4$,
    then
        $\loss_{\distb}(\hyp)
        > \loss_{\distb}(\concept) + 3 \acc / 4$.
    Hence, if $\mech$ is an algorithm
    which $(\acc/4,\conf)$-learns $\concepts$
    from $\dsize$ samples, then
        $\tv(\mech(\dista^\dsize),\mech(\distb^\dsize)) \ge 1 - 2\conf.$
\end{lemma}

\begin{proof}
The main observation is that, since $\dista$ and $\distb$ share the same marginal on $\iuni$ but the labels are given by the functions $s$ and $-s$, for any hypothesis $\hyp:\concepts\to\{\pm 1\}$ we have $\loss_\dista(\hyp) + \loss_\distb(\hyp) = 1$. Therefore,
\begin{align*}
(\loss_\dista(\hyp) - \loss_\dista(\conceptb))
+ 
(\loss_\distb(\hyp) - \loss_\distb(\concept))
&= 
((\loss_\dista(\hyp) + \loss_\distb(\hyp))
 - (1-\loss_\distb(\conceptb))- \loss_\distb(\concept)\\
&= \loss_\distb(\conceptb)- \loss_\distb(\concept)
> \alpha.
\end{align*}
This implies that if $\loss_\dista(\hyp) - \loss_\dista(\conceptb)\le \frac{\alpha}{4}$, then $\loss_\distb(\hyp) - \loss_\distb(\concept) > \frac{3\alpha}{4}$, as required.

Suppose now that $\mech$ $(\acc/4,\conf)$-learns $\concepts$ agnostically with $\dsize$ samples.
Let 
$A\subseteq \{\pm 1\}^\iuni$ be the set of hypotheses with loss at most $\loss_\dista(\conceptb) + \acc/4$ on $\dista$. As we just showed, every hypothesis in $A$ has loss larger than $\loss_\distb(\concept) + 3\acc/4$ under $\distb$. Since
\[
\min_{\concept'' \in \concepts}  \loss_\dista(\concept'') \le \loss_\dista(\conceptb),
\ \ \ \ \ \ 
\min_{\concept'' \in \concepts}  \loss_\distb(\concept'') \le \loss_\distb(\concept),
\]
it follows from the definition of agnostic learning that $\pr{}{\mech(\dista^\dsize) \in A} \ge 1-\beta$, and $\pr{}{\mech(\distb^\dsize) \in A} \le \beta$. 
Then, by the definition of total variation,
\begin{align*}
    \tv(\mech(\dista^\dsize),\mech(\distb^\dsize))
    &\ge
    \pr{}{\mech(\dista^\dsize) \in A}
    -
    \pr{}{\mech(\distb^\dsize) \in A}
    \ge 1-2\beta,
\end{align*}
completing the proof of the lemma.
\end{proof}

Finally, with these results at our disposal,
we may obtain the lower bound of
Lemma~\ref{lm:aglearnlb}.
\vspace{-0.8em}
\begin{proof}[Proof of Lemma~\ref{lm:aglearnlb}]

    Let $\dumat \in \R^{\concepts^2 \times \iuni}$,
    $\|\dumat\|_1 = 1$,
    satisfy \eqref{eq:dualinst}.
    Let $\widetilde{\distv}$,
    together with $\widetilde{\dista}_{\concept,\conceptb}$
    and $\widetilde{\distb}_{\concept,\conceptb}$
    be the distributions guaranteed to exist
    by Lemma~\ref{lm:hard-distros} and let
    $\widetilde{\dumat} \in \R^{\concepts^2 \times \iuni}$
    be the corresponding matrix
    with entries
    \[
        \widetilde{u}_{(\concept,\conceptb), \ielem}
        = \widetilde{\distv}(\concept,\conceptb) \left(
            \widetilde{\dista}_{\concept,\conceptb}(\ielem,1)
            -
            \widetilde{\distb}_{\concept,\conceptb}(\ielem,1)
        \right)
        = - \widetilde{\distv}(\concept,\conceptb) \left(
            \widetilde{\dista}_{\concept,\conceptb}(\ielem,-1)
            -
            \widetilde{\distb}_{\concept,\conceptb}(\ielem,-1)
        \right).
    \]
    Let $M$ be the matrix with entries
    $m_{(\concept,\conceptb),\ielem}
    = \widetilde{u}_{(\concept,\conceptb),\ielem} /
        \widetilde{\distv}(\concept,\conceptb)$.
    By Lemma~\ref{lm:gammastar-inftyto2-nonuniform},
    there exists some distribution $\hat{\distv}$
    with support contained in that of $\widetilde{\distv}$
    such that
    \[
        \| M \|_{\ell_\infty \to L_2(\widehat{\distv})}
        \le 4\fnorminf^\ast(\widetilde{\dumat})
        \le \frac{4\acc\fnorminf^\ast({\dumat})}{\difmat\bullet\dumat},
    \]
    where the last inequality follows from Lemma~\ref{lm:hard-distros}.
    Combining Lemma \ref{lm:kl-div} 
    with the dual formulation \eqref{eq:dualinst} then
    gives
    \begin{align*}
        \ex{(\concept,\conceptb) \sim \widehat{\distv}}{\div(\trans(\widetilde{\dista}_{\concept,\conceptb}^\dsize)\|\trans(\widetilde{\distb}_{\concept,\conceptb}^\dsize))}
        &\le O(\dsize \priv^2)\cdot \|M\|_{\ell_\infty \to L_2(\widehat{\distv})}^2 \\
        &\le
        O(\dsize \priv^2)\cdot \left(\frac{\acc\fnorminf^*(\dumat)}{\difmat\bullet\dumat} \right)^2
        \le O(\dsize \priv^2) \cdot \left(\frac{\acc}{\fnorminf(\difmat,\acc)}\right)^2.
    \end{align*}
    
    Now let
    \[
        \acc'
        = \frac{1}{4} \left( \min_{\concept, \conceptb:
            \hat{\pi}(\concept,\conceptb)> 0}
            \loss_{\distb_{\concept,\conceptb}}(\conceptb) - \loss_{\distb_{\concept,\conceptb}}(\concept)
        \right)
         \ge \frac{\acc}{O(\log(1/\acc))},
    \]
    where the last inequality is by Lemma~\ref{lm:hard-distros}.
    By Lemma~\ref{lm:tv} and Pinsker's inequality, if $\mech$
    $(\acc',\conf)$-learns $\concepts$
    for some $\conf = \frac12 - \Omega(1)$,
    then
    \(
        \ex{(\concept,\conceptb) \sim \widetilde{\distv}}{\div(\trans(\widetilde{\dista}_{\concept,\conceptb}^\dsize)\|\trans(\widetilde{\distb}_{\concept,\conceptb}^\dsize))}
        = \Omega(1).
    \)
    This implies
    \(
        \dsize
        = \Omega \left( \left( \frac{\fnorminf(\difmat,\acc)}{\priv \acc} \right)^2 \right),
    \)
    as was to be proved.
    \cut{
    It remains to bound the right-hand side. For the arguments below, recall that \(\difmat\bullet \dumat > \acc\).

    \ul{Case 1: $\difmat \bullet \dumat \le 2\acc$.}
    In this case,
    we have $\difmat \bullet \dumat - \acc \le \acc$.
    Hence, 
    \[
        \dsize
        = \Omega \left(  \frac{\fnorminf(\difmat,\acc)^2}{\priv^2 \acc^2} \right).
    \]
    Furthermore,
    \[
        \acc'
        \ge \frac{\difmat \bullet \dumat - \acc/4}{O(\log(1/\acc)} 
        \ge \Omega \left( \frac{\acc}{\log(1/\acc)}  \right)
    \]

    \ul{Case 2: $\difmat \bullet \dumat \ge 2\acc$.}
    In this case we may consider
    alternative distributions
    $
    \widehat{\distb}_{\concept,\conceptb} = (1 - \tau) \cdot \widetilde{\dista}_{\concept,\conceptb} + \tau \cdot \widetilde{\distb}_{\concept,\conceptb}
    $
    and $\widehat{\dista}_{\concept,\conceptb} = \widetilde{\dista}_{\concept,\conceptb}$, defined for all $\concept,\conceptb \in \concepts$.

    We have
    \begin{align*}
      &\div(\trans(\widehat{\dista}_{\widehat{\distv}}^n)\|\trans(\widehat{\distb}_{\widehat{\distv}}^n)) \\
      & \le O(\dsize \priv^2) \cdot \max_{f \in \R^{\iuni \times \oo}: \|f\|_\infty \le  1} \ex{{\concept,\conceptb} \sim \widehat{\distv}}{\left(\ex{(\ielem,\lab) \sim \widehat{\dista}_{\concept,\conceptb}}{f_{(\ielem,\lab)}} - \ex{(\ielem,\lab) \sim \widehat{\distb}_{\concept,\conceptb}}{f_{(\ielem,\lab)}}\right)^2} \\
      & = O(\dsize \priv^2) \cdot \tau^2 \max_{f \in \R^{\iuni \times \oo}: \|f\|_\infty \le  1} \ex{{\concept,\conceptb} \sim \widehat{\distv}}{\left(\ex{{(\ielem,\lab)} \sim \widetilde{\dista}_{\concept,\conceptb}}{f_{(\ielem,\lab)}} - \ex{(\ielem,\lab) \sim \widetilde{\distb}_{\concept,\conceptb}}{f_{(\ielem,\lab)}}\right)^2} \\
        &\le O(\dsize \priv^2) \cdot \tau^2 \cdot \left( \frac{\difmat \bullet \dumat - \acc}{\fnorminf(\difmat, \acc)} \right)^2.
    \end{align*}
    Meanwhile, we can verify
    \[
        \loss_{\hat{\distb}_{\concept,\conceptb}}(\conceptb) - \loss_{\hat{\distb}_{\concept,\conceptb}}(\concept)
        =
        \tau \cdot
        \left( \loss_{\widetilde{\distb}_{\concept,\conceptb}}(\conceptb) - \loss_{\widetilde{\distb}_{\concept,\conceptb}}(\concept) \right)
        \ge
        \tau \cdot
        \left( \frac{\difmat \bullet \dumat - \acc/4}{O(\log(1/\acc))} \right).
    \]
    As before, taking
    \[
        \acc'
        = \frac{1}{4} \left( \min_{\concept,\conceptb}
            \loss_{\hat{\distb}_{\concept,\conceptb}}(\conceptb) - \loss_{\hat{\distb}_{\concept,\conceptb}}(\concept)
        \right)
    \]
    implies, for a mechanism which
    $(\acc',\conf)$-learns $\concepts$
    with $\conf < 1/2$, then the number of samples
    required satisfies
    \[
        \dsize
        = \Omega \left( \left( \frac{\fnorminf(\dumat,\acc)}{\priv \tau \cdot (\difmat \bullet \dumat - \acc)} \right)^2 \right)
        = \Omega \left( \left( \frac{\fnorminf(\dumat,\acc)}{\priv \tau \cdot \difmat \bullet \dumat} \right)^2 \right).
    \]
    By taking
    $\tau = \acc/\difmat \bullet \dumat$,
    we get
    \[
        \dsize
        = \Omega \left( \left( \frac{\fnorminf(\dumat,\acc)}{\priv \tau \cdot \acc} \right)^2 \right).\qedhere
    \]}
\end{proof}

\cut{\subsection{Equivalence of refutation and learning}
\snote{Remove this section}
\anote{Not sure about this section. Previous version listed a number of tasks which are characterized by $\fnorminf$. The current revision focuses on just learning and refutation. However, I worry that repeating the results from \cite{factorization} takes up too much space here. I also wonder whether the equivalences should be given more explicitly.}
\tnote{I think it should either be removed or incorporated much earlier in this section.}

In \cite{factorization},
it was shown that, given a concept class $\concepts$
with concept matrix $\qmat$,
the quantity $\fnorminf(\qmat,\acc)$
gives upper bounds on the sample complexities of both
agnostic learning and agnostic refutation of $\concepts$
under non-interactive LDP.
That work also showed a lower bound for agnostic refutation
of $\concepts$ in terms of $\fnorminf(\qmat,\acc)$.
Our lower bound for agnostic learning,
given by Theorem~\ref{thm:aglearnlbqmat},
completes this picture.

By Theorem~\ref{thm:aglearnlbqmat},
a sample-complexity upper bound for agnostic learning
implies an upper bound on $\fnorminf(\qmat,\acc)$.
Together with the upper bound from \cite{factorization}
on agnostic learning in terms of $\fnorminf(\qmat,\acc)$,
this implies a sample-complexity upper bound
for agnostic refutation.
In the case where the concept class is closed under negation,
the following corollary formalizes this implication
that, under non-interactive LDP,
agnostic learnability implies agnostic refutability.
\begin{corollary}

    Let $\concepts \subseteq \oo^\iuni$ be a concept class closed under negation.
    Let $\priv > 0$, $\acc \in (0,1]$.
    Then, for some
    \[
        \acc' = 
        \Omega\left(
            \frac{\acc}{1 + \acc} 
                \middle / \log \left(
                    \frac{ 1 + \acc }{ \acc }
                \right)
        \right),
    \]
    if there exists a mechanism
    $\mech':(\iuni \times \oo)^{\dsize'} \rightarrow \oo$
    which
    $(\acc',1-\Omega(1)$-learns $\concepts$ agnostically
    with $\dsize'$ samples,
    then, for all $\theta \in [0,1]$,
    there exists a mechanism
    $\mech_\theta:(\iuni \times \oo)^\dsize \rightarrow \oo^\iuni$
    which 
    $(\acc,\conf)$-refutes $\concepts$ agnostically
    for threshold $\theta$
    with an input of size at most
    \[
        \dsize
        = O \left( \dsize' \cdot \log(|\concepts|/\conf) \right)
    \]
\end{corollary}
An analogous result may be obtained
in the case where the concept class
is not closed under negation.
Moreover, the results of \cite{factorization},
namely the upper bound on agnostic learning
and the lower bound on agnostic refutation,
allow the derivation of a similar implication
from agnostic refutation to agnostic learning.
Altogether, these results imply that,
under non-interactive LDP,
learning and refutation require
approximately the same number of samples.
Note that this equivalence is obtained
indirectly.
It remains an open problem to show how to directly
transform an algorithm for learnability
to one for refutability, and vice versa.}


\section{Refutation versus Learning: Realizable Case}
\label{sec:realizable}

In this section, we present our algorithm for realizable learning and refutation for non-interactive LDP.
For a
concept class $\concepts:\iuni\rightarrow \oo$,
we define a quantity $\mcnew(\concepts, \acc)$
and argue that it gives an upper bound
on the sample complexity for
realizable learning of $\concepts$.

\begin{definition}\label{def:mcnew}

    Let $\concepts:\iuni\rightarrow \oo$
    be a concept class.
    Let
    \begin{equation}\label{eq:qmatset}
        K_{\concepts} = \Bigl\{
        {W} \in \R^{\concepts \times (\iuni \times \oo)} \ : \
        |{w}_{\concept,(\ielem,\concept(\ielem))}| \le \acc 
        \text{ and }
        {w}_{\concept, (\ielem,-\concept(\ielem))} \ge 1  \ \forall \concept \in \concepts, \ielem \in \iuni
        \Bigr\}.
    \end{equation}
    Let
    \begin{equation}\label{eq:qmatsetgen}
        K'_{\concepts} = \left\{
            \widetilde{W} \in \R^{\concepts \times (\iuni \times \oo)} \ : \
        \exists {W} \in K_\concepts, \
        \exists \theta \in \R^\concepts, \
        \widetilde{W} = W + \theta \mathbf{1}^T
        \right\},
    \end{equation}
    where $\mathbf{1}^T$ is the all-ones row vector
    indexed over $\concepts$,
    so that
    $ \widetilde{W} = W + \theta \mathbf{1}^T $
    is the matrix obtained by shifting each row $\concept$
    of $W$ in each entry by $\theta_\concept$.

    Then define
    \anote{Do we have a name for this quantity?}
    \[
        \mcnew(\concepts, \acc) = 
        \min\Bigl\{
            \fnorminf (\widetilde{W}): \widetilde{W} \in K_{\concepts}' \Bigr
        \}.
    \]

\end{definition}

The idea is that each row of ${W}$
defines a statistical query corresponding to a concept,
$\query_\concept(\ielem,\lab)
= w_{\concept,(\ielem,\lab)} $.
The statistical query corresponding to the true concept that was used to label the data will have a small value, whereas any query corresponding to a concept with large loss will have a large value. The next theorem formalizes this argument.
\cut{Intuitively, this works
because $W$ assigns a penalty
for each labeled sample,
and the penalty is at most $\acc$
for correctly labeled samples, and at
least $1$ for incorrectly labeled ones.
Moreover, if $\widetilde{W}$
is obtained from $W \in K_{\concepts}$
by translating every row $\concept$ of $W$
by some $\theta_\concept \in \R$ in each dimension,
then answering the queries given by $\widetilde{W}$
allows us to answer the queries given by $W$ by just shifting the query answers.}

\begin{theorem}\label{thm:realub}

    Let $\concepts \subseteq \oo^\iuni$ be a concept class.
    Let $\priv > 0$, $\acc,\conf \in (0,1]$.
    Then there exists an \(\priv\)-LDP mechanism
    which may be used to both
    $(3\acc,\conf)$-learn $\concepts$ realizably
    and $(3\acc,\conf)$-refute $\concepts$ realizably
    \anote{Mechanism is for both learning and refutation. Is the phrasing here ok? Another option would be to include a definition in the prelims for a "refuting learner", i.e. a mechanism that does both.}\snote{I think it's fine as is.}
    with $\dsize$ samples, where
    \[
      \dsize = O \left(
          \frac{\mcnew(\concepts, \acc) \cdot
          \log(|\concepts|/\conf)}{\eps^2 \acc^2}
      \right).
    \]
\end{theorem}
\begin{proof}
    As per Definition~\ref{def:mcnew},
    let $\widetilde{W} \in K_{\concepts}'$ be the matrix that witnesses $\mcnew(\concepts,\acc)$
    and
    let ${W} \in K_{\concepts}$ and \(\theta \in \R^\concepts\) be the matrix and vector which
    witness $\widetilde{W} \in K_{\concepts}'$.
    If we can answer the statistical queries given by 
    $\widetilde{W}$,
    then we can answer the queries given by ${W}$
    with the same accuracy by subtracting \(\theta_\concept\) from the answer to the query for concept \(\concept\).

    By the definition of ${W}$,
    if, for some $\concept \in \concepts$,
    $\dista$ is supported
    on those $(\ielem,\lab) \in \iuni \times \oo$
    which satisfy $\concept(\ielem) = \lab$,
    then the value of the query
    corresponding to $\concept$ is bounded as
    \[
        \ex{(\ielem,\lab) \sim \dista}{
            {w}_{\concept,(\ielem, \lab)} }
        = \ex{(\ielem,\lab) \sim \dista}{
            {w}_{\concept,(\ielem, \concept(\ielem))} }
        \le \acc.
    \]
    Meanwhile,
    for an arbitrary distribution $\dista$
    on $\iuni \times \oo$,
    the value of the query corresponding to
    $\concept \in \concepts$
    may be bounded as
    \begin{equation}\label{eq:errorquery}
        \ex{(\ielem, \lab) \sim \dista}{{w}_{\concept,(\ielem, \lab)}} 
        \ge  \pr{(\ielem, \lab)  \sim \dista}{\lab \neq \concept(\ielem)} 
        - \alpha \cdot \pr{(\ielem, \lab) \sim \dista}{\lab = \concept(\ielem)}
        \ge \loss_\dista(\concept) - \alpha.
    \end{equation}
    In particular, if
    $
        \loss_\dista(\concept) \ge 3\alpha
    $,
    then
        $\ex{(\ielem, \lab) \sim \dista}{{w}_{\concept,(\ielem, \lab)}} 
        \ge 2\alpha$.


    It follows that,
    by approximating the statistical queries given by $W$
    with worst-case error $\frac{\acc}{2}$,
    we can distinguish the case where
    $\dista$ agrees with some $\concept \in \concepts$
    from the case where,
    for all concepts
    $\concept \in \concepts$,
    $
        \loss_\dista(\concept) \ge 3\alpha
    $.
    In the former case,
    returning some $\conceptb \in \concepts$
    where our estimate of
        $\ex{(\ielem, \lab) \sim \dista}{{w}_{\concept,(\ielem, \lab)}}$ 
    guarantees that it is strictly less than $2\alpha$ implies that
    $
        \loss_\dista(\concept) < 3\alpha
    $.

    To complete the proof,
    it suffices to apply the upper bound from
    \cite{factorization}
    which says that,
    to answer the collection of statistical queries
    given by $\widetilde{\qmat}$
    under non-interactive $\priv$-LDP,
    with accuracy $\frac{\acc}{2}$
    and probability of failure at most $\conf$,
    the number of samples required is at most
    \[
        O \left(
            \frac{\fnorminf(\widetilde{\qmat})\log(|\concepts|/\conf)}{\eps^2 \acc^2}
        \right)
        =
        O \left(
            \frac{\mcnew(\concepts,\acc)\log(|\concepts|/\conf)}{\eps^2 \acc^2}
        \right).\qedhere
    \]
\end{proof}

\subsection{Lower bound}


Our lower bound will follow a similar strategy as in the agnostic case. However, our construction of hard distributions will be tailored to \(\mcnew(\concepts,\acc)\) and its dual.
\cut{construct
families $\{\dista_1,\dots,\dista_\qsize\}$ and
$\{\distb_1,\dots,\distb_\qsize\}$ of distributions on $\iuni = \iuni \times \oo$
as well as a distribution $\distv$ over ${\qsize}$.
For any $v \in [\qsize]$, let $\dista_v^\dsize$ be the product
distribution induced by sampling $n$ times independently from
$\dista_v$.\cut{, and let $\dista_\distv^\dsize$ be the mixture $\sum_{v =
  1}^{\qsize}\distv(v) \dista^\dsize_v$. Define $\distb_v^\dsize$ and
$\distb_\distv^\dsize$ analogously. Note that $\dista_\distv^\dsize$ and
$\distb_\distv^\dsize$ are \emph{not} product distributions, but mixtures
of such distributions. }

Suppose our distributions met the following criteria:
\begin{equation}\label{eq:fit}
    \forall v \in [\qsize], \
    \exists \concept \in \concepts, \
    \loss_{\dista_v}(c) = 0
\end{equation}
\begin{equation}\label{eq:nofit}
    \forall v \in [\qsize], \
    \forall \hyp:\iuni \to \{\pm 1\} \
    \loss_{\distb_v}(\hyp) \ge \acc
\end{equation}
Then an algorithm $\mech$ able to
$(\acc,\conf)$-refute $\concepts$ realizably
is able to distinguish
$\dista_\distv^\dsize$ from
$\distb_\distv^\dsize$ with probability $1 - \conf$.
When $1 - \conf = \Omega(1)$,
this implies
\[
    \tv(\trans(\dista_v^\dsize)\|\trans(\distb_v^\dsize))
    \ge \Omega(1),
\]
and hence
\[
    \div(\trans(\dista_v^\dsize)\|\trans(\distb_v^\dsize))
    \ge \Omega(1).
\]
This allows us again to apply Lemma~\ref{lm:kl-div},
giving the sample complexity lower bound
\[
    \dsize =
    \Omega \left(
        \frac{ 1 }{ \priv^2 \cdot \|M\|_{\ell_\infty \to L_2(\distv)}^2 }
    \right).
\]
where
$M \in \R^{[\qsize] \times \iuni}$
is the matrix with entries
$m_{v, (\ielem,\lab)}
= \dista_v(\ielem,\lab) - \distb_v(\ielem,\lab)$.
This motivates us to define
the distributions
$\dista_1,\dots,\dista_\qsize$,
$\distb_1,\dots,\distb_\qsize$,
and $\distv$
so as to guarantee \eqref{eq:fit}
\eqref{eq:nofit} while,
at the same time,
keeping $\|M\|_{\ell_\infty \to L_2(\distv)}^2$
as small as possible.
}
\subsubsection{Duality}

We will again use convex duality in our lower bound. We will express $\mcnew(\concepts, \acc)$ as a maximum over dual matrices $U$, and we will use an optimal $U$ to construct
`hard distributions' for realizable refutation.
To this end, consider the following duality lemma,
proved in Appendix~\ref{ap:realizable-lb}.

\begin{lemma}\label{lm:realizable-dual}

    For any concept class
    $\concepts \subseteq \oo^\iuni$
    and any $\acc$, 
    \begin{equation}\label{eq:realizabledual}
    \mcnew(\concepts, \acc) 
    = 
    \max_{U \in S_\concepts} \frac
      {
          \sum_{\concept \in \concepts, \ielem \in \iuni} ( u_{\concept, (\ielem,  -\concept(\ielem))}
          - \acc |u_{\concept, (\ielem,  \concept(\ielem))}| )
      }
      {\fnorminf^\ast(U)},
    \end{equation}
    where we define
    \begin{align*}
        S_\concepts := \bigg\{ U \in \R^{\concepts \times (\iuni \times \oo)} \ : \
            \forall \concept \in \concepts, \
            \sum_{\ielem \in \iuni} ( u_{\concept,(\ielem,\concept(\ielem))} + u_{\concept,(\ielem,-\concept(\ielem))} ) = 0 &
            \\ \text{and, }
            \forall \concept \in \concepts, \forall \ielem \in \iuni, \
            u_{\concept,(\ielem,-\concept(\ielem))} \ge 0
        &\bigg\}.
    \end{align*}

\end{lemma}

\subsubsection{Hard distributions}\label{sec:realizable-hard}

Let $U \in \R^{\concepts \times (\iuni \times \oo)}$ witness \eqref{eq:realizabledual}.
By normalizing, we may assume without loss of generality that $\|U\|_1 = 1$.
We will consider the matrices $U^+, U^- \in \R^{m \times N}$
with non-negative entries which satisfy $U = U^+ - U^-$
so that $U^+$ and $U^-$ correspond to the positive and negative
entries of $U$ respectively.
We define the distribution $\distv$ on $\concepts$ given by
    $\distv(\concept) = \sum_{(\ielem,\lab) \in \iuni \times \oo} |u_{\concept,(\ielem,\lab)}|.$
Then, for each $\concept \in \concepts$,
let $\dista_\concept$ and $\distb_\concept$ be the distributions on $\iuni \times \oo$
given by
\[
    \dista_\concept(\ielem,\lab)
    = \frac{2u^+_{\concept,(\ielem,\lab)}}{\distv(\concept)} 
    \qquad\text{and}\qquad
    \distb_\concept(\ielem,\lab)
    = \frac{2u^-_{\concept,\ielem,\lab}}{\distv(\concept)}.
\]
Since the rows of $U$ each sum to zero
and the row corresponding to $\concept$ has 
 $\ell_1$ norm $\distv(\concept)$,
the distributions $\dista_\concept$ and $\distb_\concept$ are well-defined.
Moreover,
since
$u_{\concept,(\ielem,-\concept(\ielem))} \ge 0$ for all $\concept \in \concepts, \ielem \in \iuni$,
the only negative entries of $U$ are those of the form $u_{\concept,(\elem,\concept(\elem))}$.
This implies that the distribution $\distb_\concept$
always labels samples $\ielem \in \iuni$
by $\concept(\ielem)$.

\subsubsection{Warm-up: single-concept case}
\label{sec:warmup}


Consider the case where $\concepts$ consists of a single concept
$\concept$. 
Since $\mcnew(\concepts,\acc) > 0$, 
then \eqref{eq:realizabledual} implies
\begin{equation}\label{eq:sepcriteria}
    \sum_{\ielem \in \iuni} u_{\concept,(\ielem,-\concept(\ielem))}
     > \sum_{\ielem \in \iuni} \acc | u_{\concept,(\ielem,\concept(\ielem))} |.
\end{equation}
Hence,
\[
    \pr{(\ielem,\lab) \sim \dista_\concept}{\concept(\ielem) \ne \lab}
    -
    \pr{(\ielem,\lab) \sim \distb_\concept}{\concept(\ielem) \ne \lab}
    > \acc \cdot \left(
        \pr{(\ielem,\lab) \sim \dista_\concept}{\concept(\ielem) = \lab}
        +
        \pr{(\ielem,\lab) \sim \distb_\concept}{\concept(\ielem) = \lab}
    \right).
\]
Using
\begin{equation}\label{eq:realnofit}
    \pr{(\ielem,\lab) \sim \distb_\concept}{\concept(\ielem) = \lab} = 1
\end{equation}
and rearranging, this gives
\[
    \loss_{\dista_\concept}(\concept)
    = \pr{(\ielem,\lab) \sim \dista_\concept}{\concept(\ielem) \ne \lab}
    > \frac{2 \acc}{1 + \acc}.
\]
In other words,
if we can distinguish a distribution on $\iuni \times \oo$
which labels samples according to $\concept$
from one which disagrees with $\concept$ with probability greater than
$\frac{2\acc}{1 + \acc}$,
then we can distinguish between $\dista_\concept$ and $\distb_\concept$.

\subsubsection{General case}

\cut{%
Appendix~\ref{ap:realizablegencase}
generalizes the lower bound of the previous section
to the general case
where the concept class is not restricted
to a single concept.

The first issue which needs to be addressed
in the general case
is that \eqref{eq:sepcriteria},
rather than holding in worst case over all concepts,
holds on average for a concept
$\concept \in \concepts$
drawn from the distribution $\distv$.
This issue is handled by
applying the binning result of Lemma~\ref{lm:binning}.

The second issue which needs to be addressed
is that, while each $\concept \in \concepts$
is guaranteed not to fit the corresponding distribution
$\dista_\concept$, as with \eqref{eq:realnofit},
it may hold that some other $\hyp:\iuni \rightarrow \oo$
has small loss on $\dista_\concept$.
This is remedied by mixing
a distribution $\sigma_\concept$
which agrees with $\concept$
into the distribution $\dista_\concept$.
This guarantees that every
$\hyp:\iuni \rightarrow \oo$
has large loss on $\sigma_\concept$.

Altogether, we obtain the lower bound of
Theorem~\ref{thm:agnostic-main}.
}%


While Section~\ref{sec:warmup}
demonstrated how our lower bound for realizable refutation
can be derived in the single-concept case,
there are two issues to resolve in the general case:
\begin{enumerate}
    \item
        Instead of equation \eqref{eq:sepcriteria}
        holding for each concept, it holds on average.
        In particular,
        \[
            \sum_{\concept \in \concepts, \ielem \in \iuni} u_{\concept,{(\ielem,-\concept(\ielem))}}
             > \sum_{\concept \in \concepts, \ielem \in \iuni} \acc | u_{\concept,(\ielem,\concept(\ielem))} |.
        \]
        Equivalently,
        \begin{equation}\label{eq:exlb}
            \E_{\concept \sim \distv} \left[
                \loss_\dista(\concept)
            \right]
            > \frac{2 \acc}{1 + \acc}.
        \end{equation}
    \item
        Even if we can guarantee for a concept
        $\concept \in \concepts$ that
        \[
            \quad
            \loss_{\dista_\concept}(\concept)
            > \frac{2 \acc}{1 + \acc},
        \]
        it may hold, for some other
        $\hyp:\iuni \to \{\pm 1\}$, that
        $
            \loss_{\dista_\concept}(\hyp)
        $
        is small. We need to rule out this possibility in order to give a lower bound against refutation. 
\end{enumerate}

\vspace{.5 em}
The first issue is resolved in Lemma~\ref{lm:bincor}
by applying the binning result of Lemma~\ref{lm:binning}.
The second issue will be resolved in Lemma~\ref{lm:allcase}.

\begin{lemma}\label{lm:bincor}

    Suppose there exist families
    $\{\dista_\concept\}_{\concept \in \concepts}$
    and
    $\{\distb_{\concept}\}_{\concept \in \concepts}$
    of distributions over $\iuni$,
    together with a parameter distribution
    $\distv$ over $\concepts$,
    such that
    \[
         \Delta = \E_{\concept \sim \distv} \left[ \loss_{\dista_\concept}(\concept) \right]
        > \frac{2\acc}{1 + \acc} 
    \]
    while, for all $\concept \in \concepts$,
    $
    \loss_{\distb_\concept}(\concept)
    = 0.
    $
    Let, further, \(\dumat \in \R^{\concepts \times \iuni}\) be
    the matrix with entries
    \(
    {u}_{\concept,\ielem}
      = {\distv}(\concept) ( {\dista}_\concept(\ielem) - 
      {\distb}_\concept(\ielem) ).
    \)

    Then there exist families
    $\{\widetilde{\dista}_\concept\}_{\concept \in \concepts}$
    and
    $\{\widetilde{\distb}_{\concept}\}_{\concept \in \concepts}$
    of distributions over $\iuni \times \oo$,
    together with a parameter distribution
    $\widetilde{\distv}$ over $\concepts$,
    such that
    \begin{enumerate}
    \item for all $\concept$ in the support of $\widetilde{\distv}$,
    \[
      \loss_{\widetilde{\dista}_\concept}(\concept)
      \ge
      \Omega\left(
        {\frac{\acc}{1+\acc}}\middle /{\log \left(
            \frac{ 1 + \acc }{ \acc }
          \right)}
      \right);
    \]

  \item 
    $
        \loss_{\widetilde{\distb}_\concept}(\concept)
        = 0
    $
    for all $\concept \in \concepts$;

  \item the matrix
    $
        \widetilde{U} \in \R^{\concepts \times \iuni}
    $
    with entries 
    $
        \widetilde{u}_{\concept,\ielem}
        = \widetilde{\distv}(\concept) ( \widetilde{\dista}_\concept(\ielem) - 
            \widetilde{\distb}_\concept(\ielem) )
    $
    satisfies
    \(
        \fnorminf^\ast(\widetilde{U})
        \le \frac{2\acc \fnorminf^\ast(U)}{(1+\acc) \Delta}.
    \)
  \end{enumerate}
\end{lemma}

\begin{proof}

    Apply Lemma \ref{lm:binning}, with
    $
        a_\concept =
        \loss_{\dista_\concept}(\concept)
    $
    for all $\concept \in \concepts$, and \(\conf = \frac{\acc}{1+\acc} < \frac{\Delta}{2}\),
    to obtain $S \subseteq \concepts$ such that
    \[
        \distv(S) \cdot \min_{\concept \in S} a_\concept
        \ge
        \frac{\Delta - \conf}{O(\log(1 / \conf)}
        \ge
        \frac{\Delta}{O(\log((1+\acc)/\acc)}.
      \]   
    Let $\widetilde{\distv}$ be $\distv$ conditional
    on membership in $S$. Thus,
    \[
        \widetilde{\distv}(v)
        =
        \begin{cases}
            \distv(v)/\distv(S), &\text{if }v \in S \\
            0, &\text{otherwise.}
        \end{cases}
    \]
    Let
    \(
    \tau = \frac{2\acc}{(1+\acc) \Delta} \in (0,1).
    \)
    For all $\concept \in \concepts$,
    define $\widetilde{\distb}_{\concept} = \distb_{\concept}$
    and $\widetilde{\dista}_{\concept} =
    \tau \distv(S)\dista_{\concept} +
    (1-\tau\distv(S))\distb_{\concept}$. 
    Then, for all $\concept$ in the support of $\widetilde{\distv}$,
    \[
        \loss_{\widetilde{\distb}_{\concept}} ( \concept)
        =
        \loss_{\distb_{\concept}} ( \concept)
        = 0
    \]
    \[
        \loss_{\widetilde{\dista}_{\concept}}(\concept)
        =
        \tau \cdot \distv(S) \cdot
        \loss_{\dista_{\concept}} ( \concept)
        \ge
        \frac{\frac{\acc}{1+\acc}}{O(\left(\log \left(
                   \frac{ 1 + \acc }{ \acc }
                 \right)\right)}.
           \]

    Moreover, from the construction of the matrix 
    $
    \widetilde{U}
    $
    and the definition of the dual norm $\fnorminf^\ast$,
    it follows immediately that
    $
        \fnorminf^\ast(\widetilde{U})
        \le \fnorminf^\ast(\tau U) = \frac{2\acc\fnorminf^\ast(U)}{(1+\acc)\Delta}.
    $
\end{proof}

\begin{lemma}\label{lm:allcase}
    Suppose we have distributions
    $\dista_\concept$ and $\mu_\concept$
    on $\iuni \times \oo$
    for each $\concept \in \concepts$
    where:
    \begin{enumerate}
        \item 
            $ \loss_{\distb_\concept}(\concept) = 0 $;
        \item
            $ \loss_{\dista_\concept}(\concept) > \acc $.
        \setcounter{listcounter}{\value{enumi}}
    \end{enumerate}
    Then there exist distributions
    $\widetilde{\dista}_\concept$ and $\widetilde{\distb}_\concept$
    for each $\concept \in \concepts$
    such that:
    \begin{enumerate}
        \setcounter{enumi}{\value{listcounter}}
        \item\label{it:distpropa}
            $ 
                \loss_{\widetilde{\distb}_\concept}(\concept)
                = 0
            $;
        \item\label{it:distpropb}
            $\forall \hyp : \iuni \rightarrow \oo$, \
            $
                \loss_{\widetilde{\dista}_\concept}(\hyp) > \frac{\acc}{2};
            $
        \item\label{it:distpropc}
            $
                \widetilde{\dista}_\concept - \widetilde{\distb}_\concept
                = \frac{1}{2} ( \dista_\concept - \distb_\concept )
            $.
    \end{enumerate}
\end{lemma}

\begin{proof}

    For $\concept \in \concepts$, let $\sigma_c$ be
    the distribution on $\iuni \times \oo$
    which has the same marginal on $\iuni$ as does $\lambda_\concept$,
    and which satisfies $\concept(\ielem) = \lab$ for all
    $(\ielem,\lab)$ in the support of $\sigma_\concept$.
    Also, let
    $\widetilde{\dista}_\concept = \frac{1}{2} \dista_\concept + \frac{1}{2} \sigma_\concept$
    and
    $\widetilde{\distb}_\concept = \frac{1}{2} \distb_\concept + \frac{1}{2} \sigma_\concept$.
    Properties \ref{it:distpropa} and \ref{it:distpropc}
    follow immediately.

    To establish property  \ref{it:distpropb}, notice first that, for any \(\ielem \in \iuni\) in the support of \(\dista_\concept\), \(\pr{\widetilde{\dista}_\concept}{b=c(a) \ | \ \ielem} \ge \frac12\), and also \(\pr{\widetilde{\dista}_\concept}{b\neq c(a) \ | \ \ielem} = \frac12\cdot\pr{{\dista}_\concept}{b\neq c(a) \ | \ \ielem}\). Then, for any function
    $\hyp:\iuni \rightarrow \oo$.
    \begin{align*}
        \loss_{\widetilde{\dista}_\concept}(\hyp) &=
        \pr{\widetilde{\dista}_\concept}{\hyp(\ielem) \ne \lab} \\
        &= \sum_{(\ielem,\lab) \in \iuni \times \oo} \pr{\widetilde{\dista}_\concept}{\ielem}
            \cdot \pr{\widetilde{\dista}_\concept}{\hyp(\ielem) \ne \lab \ | \ \ielem} \\
        &\ge \sum_{(\ielem,\lab) \in \iuni \times \oo} \pr{\widetilde{\dista}_\concept}{\ielem}
            \cdot \min \left\{
                \pr{\widetilde{\dista}_\concept}{\lab = \concept(\ielem) \ | \ \ielem},
                \pr{\widetilde{\dista}_\concept}{\lab \ne \concept(\ielem) \ | \ \ielem}
            \right\} \\
        &\ge \sum_{(\ielem,\lab) \in \iuni \times \oo} \pr{{\dista}_\concept}{\ielem}
            \cdot \min \left\{
                \frac{1}{2},
                \frac{1}{2} \cdot \pr{\dista_\concept}{\lab \ne \concept(\ielem) \ | \ \ielem}
            \right\} \\
        \cut{&= \frac{1}{2} \cdot \sum_{(\ielem,\lab) \in \iuni \times \oo} \pr{\widetilde{\dista}_\concept}{\ielem}
            \cdot \min \left\{
                1,
                \pr{\dista_\concept}{\lab \ne \concept(\ielem) \ | \ \ielem}
            \right\} \\}
        &= \frac{1}{2} \cdot \sum_{(\ielem,\lab) \in \iuni \times \oo} \pr{\dista_\concept}{\ielem}
            \cdot \pr{\dista_\concept}{\lab \ne \concept(\ielem) \ | \ \ielem}
             \\
        &= \frac{1}{2} \cdot \loss_{\dista_\concept}(\concept) \\
        &> \frac{\acc}{2}.\qedhere
    \end{align*}

\end{proof}

\cut{Finally, we will use the following lemma from 
\cite{factorization} 
to get our main result.

\begin{lemma}[Lemma 15 of \cite{factorization}]\label{lm:projection}

    Given distributions $\dista_\concept$ and $\distb_\concept$
    on $\iuni \times \oo$
    for each $\concept \in \concepts$,
    together with a parameter distribution $\distv$ on $\concepts$,
    let $U \in \R^{\concepts \times (\iuni \times \oo)}$
    be the matrix with entries
    $
        u_{\concept,(\ielem,\lab)}
        = \distv(\concept) (\dista_\concept(\ielem,\lab) - \distb_\concept(\ielem,\lab)).
    $
    Then there exists a distribution $\widetilde{\distv}$
    on $\concepts$
    with support contained in that of $\distv$
    such that the matrix
    $\widetilde{M} \in \R^{\concepts \times (\iuni \times \oo)}$
    with entries
    $
        \widetilde{m}_{\concept,(\ielem,\lab)}
        = \dista_\concept(\ielem,\lab) - \distb_\concept(\ielem,\lab)
    $
    satisfies
     \[
        \| \widetilde{M} \|_{\ell_\infty \to L_2(\hat{\distv})}
        \le 4\fnorminf^\ast(U).
     \]

\end{lemma}}

Equipped with Lemmas~\ref{lm:bincor}~and~\ref{lm:allcase}, we are
ready to prove our lower bound against realizable refutation.

\begin{proof}[Proof of Theorem~\ref{thm:realizable-main}]

    We define the parameter distribution \(\distv\) over \(\concepts\), and the distribution families
    $\{{\dista}_\concept\}_{\concept \in \concepts}$
    and
    $\{{\distb}_{\concept}\}_{\concept \in \concepts}$ over $\iuni \times \oo$,
    as in Section~\ref{sec:realizable-hard}. We denote
    \(
        \Delta = \E_{\concept \sim \distv} \left[
            \loss_{\dista_\concept}(\concept)
        \right].
    \)
    By equation \eqref{eq:exlb},
    together with Lemmas \ref{lm:bincor} and \ref{lm:allcase},
    we obtain modified families of distributions
    $\{\widetilde{\dista}_\concept\}_{\concept \in \concepts}$
    and
    $\{\widetilde{\distb}_{\concept}\}_{\concept \in \concepts}$,
    together with a parameter distribution
    $\widetilde{\distv}$ over $\concepts$,
    such that, for all $\concept$ in the support of
    $\widetilde{\distv}$, and
    for all functions $\hyp:\iuni \to \{\pm 1\}$,
    \[
        \loss_{\widetilde{\dista}_\concept}(\hyp)
        =
        \Omega\left(
               \frac{\acc}{1 + \acc} 
               \middle / \log \left(
                   \frac{ 1 + \acc }{ \acc }
               \right)
           \right)
    \]
    while
    $
        \loss_{\widetilde{\distb}_\concept}(\concept)
        = 0
    $
    for all $\concept \in \concepts$.
    By Lemmas \ref{lm:gammastar-inftyto2-nonuniform}, \ref{lm:bincor} and \ref{lm:allcase},
    we may assume further that the matrix
    $\widetilde{M} \in \R^{\concepts \times (\iuni \times \{\pm 1\})}$
    with entries
    $
        m_{\concept,(\ielem,\lab)}
        = \widetilde{\dista}_\concept(\ielem,\lab) - \widetilde{\distb}_\concept(\ielem,\lab)
    $
    satisfies
    \[
      \| \widetilde{M} \|_{\ell_\infty \to L_2(\hat{\distv})}
      \le \frac{4\acc\fnorminf^\ast(U)}{(1+\acc)\Delta}
    \]
    for some distribution $\hat{\distv}$ on $\concepts$ whose support
    is contained in the support of $\widetilde{\distv}$.
    
    Now let $\mech$ be an $\eps$-LDP protocol
    that solves the $(\acc',\beta)$-refutation problem for the concept
    class $\concepts$ in the realizable case, where we choose a small enough
    $
    \acc' = \Omega\left(
               \frac{\acc}{1 + \acc} 
               \middle / \log \left(
                   \frac{ 1 + \acc }{ \acc }
               \right)
           \right).
    $
    Then, for every \(\concept \in \concepts\) in the support of
    \(\widetilde{\distv}\), as long as $\beta = \frac12 - \Omega(1)$, by Pinsker's inequality we have
    \begin{equation}\label{eq:divbound}
      \div(\trans(\widetilde{\dista}_{\concept}^n)\|\trans(\widetilde{\distb}_{\concept}^n))
             = \Omega(1).
    \end{equation}
    Meanwhile, Lemma~\ref{lm:kl-div} guarantees
    \[
        \ex{\concept \sim\hat{\distv}}{\div(\trans(\widetilde{\dista}_\concept^\dsize)\|\trans(\widetilde{\distb}_\concept^\dsize))} \le O(n \priv^2)\cdot 
        \|\widetilde{M}\|_{\ell_\infty \to L_2(\hat{\distv})}^2,
    \]
    whereby we obtain
    \begin{equation}\label{eq:realizable-sc-lb}
        n
        = \Omega \left(
            \frac{1}{\eps^2 \cdot \|\widetilde{M}\|_{\ell_\infty \to L_2(\hat{\distv})}^2 }  
        \right)
        = \Omega \left(
            \frac{(1+\acc)^2\Delta^2}{\eps^2 \acc^2\fnorminf^\ast(U)^2 }  
        \right).
    \end{equation}
    Now we may use
    \begin{equation}\label{eq:realizable-gamma2star}
      \fnorminf^\ast(U)
      = \frac
      {
        \sum_{\concept \in \concepts,\ielem \in \iuni}
        ( u_{\concept,(\ielem,-\concept(\ielem))} - \acc | u_{\concept,(\ielem,\concept(\ielem))} | ) 
      }
      {\mcnew(\concepts, \acc)}.
    \end{equation}
    Note that, for any \(\concept \in \concepts\), since
    \(\loss_{\distb_\concept}(\concept) = 0\),
    \begin{align*}
      \frac{1}{\distv(\concept)}\sum_{\ielem \in \iuni}
         u_{\concept,(\ielem,-\concept(\ielem))}
         - \acc | u_{\concept,(\ielem,\concept(\ielem))} | 
      &=
        \pr{(\ielem,\lab) \sim \dista_\concept}{\concept(\ielem) \ne \lab}
    -
    \pr{(\ielem,\lab) \sim \distb_\concept}{\concept(\ielem) \ne \lab} \\
      &\qquad\qquad
        - \acc \cdot \left(
        \pr{(\ielem,\lab) \sim \dista_\concept}{\concept(\ielem) = \lab}
        +
        \pr{(\ielem,\lab) \sim \distb_\concept}{\concept(\ielem) = \lab}
        \right)\\
      &=
        (1+\acc) \cdot \pr{(\ielem,\lab) \sim \dista_\concept}{\concept(\ielem) \ne \lab}
        - 2\acc \\
      &=
        (1+\acc) \cdot \loss_{\dista_\concept}(\concept)
        - 2\acc.
    \end{align*}
    Taking expectations over \(\concept \sim \distv\), we have
    \begin{align}
      \sum_{\concept \in \concepts,\ielem \in \iuni}
       u_{\concept,(\ielem,-\concept(\ielem))} - \acc |
      u_{\concept,(\ielem,\concept(\ielem))} | 
      &=
      (1+\acc) \cdot \ex{\concept\sim\distv}{\loss_{\dista_\concept}(\concept)} - 2\acc \notag \\
      &=
      (1+\acc) \cdot \Delta - 2\acc. \label{eq:realizable-Delta}
    \end{align}
    Putting equations
    \eqref{eq:realizable-sc-lb}, \eqref{eq:realizable-gamma2star}, and
    \eqref{eq:realizable-Delta} together, we have
    \[
      \dsize
      =
      \Omega \left(
        \frac{(1+\acc)^2\Delta^2\mcnew(\concepts,\acc)^2}{\priv^2 \acc^2((1+\acc)\Delta-2\acc)^2}  
      \right)
      =
      \Omega \left(
        \frac{\mcnew(\concepts,\acc)^2}{\priv^2 \acc^2}  
      \right).\qedhere
    \]
\end{proof}

\cut{
As a corollary of Theorem~\ref{thm:realub}
and Theorem~\ref{thm:reallb},
it follows that realizable refutability
implies realizable learnability.
In particular, by Theorem~\ref{thm:reallb},
a sample complexity upper bound for realizable refutability
of concept class $\concepts$
gives an upper bound on $\mcnew(\concepts,\acc)$.
Then the sample complexity upper bound of
Theorem~\ref{thm:realub} in terms of
$\mcnew(\concepts,\acc)$
gives an upper bound on the number of samples required
for realizable learning of $\concepts$.

\begin{corollary}

    Let $\concepts \subseteq \oo^\iuni$ be a concept class.
    Let $\priv > 0$, $\acc \in (0,1]$.
    Then, for some
    \[
        \acc' = 
        \Omega\left(
            \frac{\acc}{1 + \acc} 
                \middle / \log \left(
                    \frac{ 1 + \acc }{ \acc }
                \right)
        \right),
    \]
    if there exists a mechanism
    $\mech':(\iuni \times \oo)^{\dsize'} \rightarrow \oo$
    which
    $(\acc',1-\Omega(1)$-refutes $\concepts$ realizably
    with $\dsize'$ samples,
    then there exists a mechanism
    $\mech:(\iuni \times \oo)^\dsize \rightarrow \oo^\iuni$
    which 
    $(\acc,\conf)$-learns $\concepts$ realizably
    with sample size $\dsize
        = O \left( \dsize' \cdot \log(|\concepts|/\conf) \right).$

\end{corollary}}


\section{Open problems}


This work, together with \cite{factorization},
largely completes the picture
of agnostic refutability and learnability 
under non-interactive LDP.
In the realizable setting, we have
shown that refutation implies learning for non-interactive LDP.
It is an interesting open problem to determine the converse -- whether realizable learning implies refutation.
Secondly,
for an arbitrary concept class $\concepts$,
can we obtain a characterization of realizable learnability
in terms of a quantity which is efficiently computable
from the definition of $\concepts$?
Furthermore,
for both the realizable and agnostic versions,
the relationships obtained between the sample complexities
of refutability and learnability in this work are indirect, via characterizations
of these tasks by the approximate $\gamma_2$ norm.
For example,
although realizable refutability implies
realizable learnability under non-interactive LDP,
it remains open how one might obtain
a non-interactive LDP protocol for learnability
directly from one for refutability.
\tnote{I removed question about finite concept classes. Can add back if we have space.}

\bibliographystyle{alpha}
\bibliography{../agnostic-realizable-ldp-paper}

\appendix

\section{\texorpdfstring{Equivalence of approximate $\fnorminf$ norms of difference and concept matrices}{Equivalence of approximate gamma2 norms of difference and concept matrices}}
\label{ap:difmatqmat}

We prove Lemma \ref{lm:normrelation-main}, the equivalence of the approximate $\fnorminf$ norm for matrices $W$ and $D$, via the following three Lemmas. 

\begin{lemma}\label{lm:normrelation-ub}

    Let $\concepts$ be a concept class
    with concept matrix
    $\qmat \in \R^{\concepts \times \iuni}$
    and difference matrix
    $\difmat \in \R^{\concepts^2 \times \iuni}$.
    Then
    $\fnorminf(\difmat,\acc)
    \leq \fnorminf(\qmat,\acc)$. 
\end{lemma}

\begin{lemma}\label{lm:normrelation-closed}
    Let $\concepts$ be a concept class
    closed under negation.
    Let
    $\qmat \in \concepts \times \iuni$
    be its concept matrix
   (Definition~\ref{def:cmatdef})
   and let
    $\difmat \in \R^{\concepts^2 \times \iuni}$
    be its difference matrix
    (Definition~\ref{def:difmatdef}).
    Then
    $\fnorminf(\qmat,\acc)
    \le \fnorminf(\difmat,\acc)$.
\end{lemma}

\begin{lemma}\label{lm:normrelation-notclosed}
    Let $\concepts$ be a concept class.
    Let
    $\qmat \in \R^{\concepts \times \iuni}$
    be its concept matrix
    (Definition~\ref{def:cmatdef})
    and let
    $\difmat \in \R^{\concepts^2 \times \iuni}$
    be its difference matrix
   (Definition~\ref{def:difmatdef}).
    Then
    $\fnorminf(\qmat,\acc)
    \le 2\fnorminf(\difmat,\acc/2) + 1$.
\end{lemma}

\begin{proof}[Proof of Lemma~\ref{lm:normrelation-ub}]

    Let $\widetilde{\qmat} \in \concepts \times \iuni$ 
    witness
    $\fnorminf(\qmat,\acc)$ so that
    $\|\qmat - \widetilde{\qmat}\|_{1\to\infty} \le \acc$
    and
    $\fnorminf(\qmat,\acc)
    = \fnorminf(\widetilde{\qmat})$.

    Let $\qmat' \in \R^{\concepts^2 \times \iuni}$
    be the matrix with entries
    $
        \qent_{(\concept,\conceptb),\ielem}'
        = \widetilde{\qent}_{\concept,\ielem}
    $.
    Similarly, let $\qmat'' \in \R^{\concepts^2 \times\iuni}$
    be the matrix with entries
    $
        \qent_{(\concept,\conceptb),\ielem}''
        = \widetilde{\qent}_{\conceptb,\ielem}
    $.
    Since $\qmat'$ and $\qmat''$
    are obtained from $\widetilde{\qmat}$ by duplicating rows,
    \[
        \fnorminf(\qmat,\acc)
        =\fnorminf(\widetilde{\qmat})
        = \fnorminf(\qmat')
        = \fnorminf(\qmat'').
    \]
    Now consider the matrix
    $
        \widetilde{\difmat}
        = \frac{1}{2} (\qmat' - \qmat'').
    $
    By subadditivity and scaling properties,
    \[
        \fnorminf(\widetilde{\difmat})
        \le \frac{1}{2} (\fnorminf(\qmat') + \fnorminf(\qmat''))
        = \fnorminf(\widetilde{\qmat}).
    \]
    Moreover, for all
    $\concept,\conceptb \in \concepts$,
    $\ielem \in \iuni$,
    the entry
    $\widetilde{d}_{(\concept,\conceptb),\ielem}$ of $\widetilde{\difmat}$
    approximates entry ${d}_{(\concept,\conceptb),\ielem}$ of $\difmat$.
    Specifically,
    \begin{align*}
        \left| \widetilde{d}_{(\concept,\conceptb),\ielem}
            -
        d_{(\concept,\conceptb),\ielem} \right|
        &=
        \left|
            \frac{\widetilde{\qent}_{\concept,\ielem} - \widetilde{\qent}_{\conceptb,\ielem}}{2} 
            -
            \frac{\concept(\ielem) - \conceptb(\ielem)}{2} 
        \right| \\
        &\le
        \left|
            \frac{\widetilde{\qent}_{\concept,\ielem} - \concept(\ielem)}{2} 
        \right|
        +
        \left|
            \frac{\widetilde{\qent}_{\conceptb,\ielem} - \conceptb(\ielem)}{2} 
        \right| \\
        &\le \acc.
    \end{align*}
    Hence $\|\difmat - \widetilde{\difmat}\|_{1\to\infty} \le \acc$.
    Together with
    $
        \fnorminf(\widetilde{\difmat})
        \le
        \fnorminf(\widetilde{\qmat})
    $,
    this implies
    \[
        \fnorminf(\difmat,\acc)
        \le
        \fnorminf(\widetilde{\difmat})
        \le
        \fnorminf(\widetilde{\qmat})
        = \fnorminf(\qmat,\acc).\qedhere
    \]
\end{proof}

\begin{proof}[Proof of Lemma~\ref{lm:normrelation-closed}]
    Row $\concept$ of $\qmat$
    is identical with row $(\concept,-\concept)$ of $\difmat$.
    Hence, $\qmat$ is obtained from $\difmat$
    by deleting some of its rows.
    Since the \(\fnorminf\) norm is non-increasing under taking submatrices, it follows that
    $\fnorminf(\qmat,\acc)
    \le \fnorminf(\difmat,\acc)$.
\end{proof}

\begin{proof}[Proof of Lemma~\ref{lm:normrelation-notclosed}]
    Fix an arbitrary concept $\conceptb \in \concepts$.
    Let $\difmat'$ be the submatrix of $\difmat$
    which includes row $(\concept,\conceptb)$ of $\difmat$
    for each $\concept\in \concepts$.
    Then
    $
        \difmat'
        = \frac{1}{2} \left(
            \qmat - \mathbf{1} (\conceptb)^T
        \right)
    $
    where $\mathbf{1}$ is the all-ones vector
    of dimension $|\concepts|$, and we identify \(\conceptb\) with a vector in \(\R^\iuni\).
    Expressing our concept matrix as
    $\qmat = 2 \difmat' + \mathbf{1} (\conceptb)^T$,
    we may apply
    the scaling and subadditivity properties
    \anote{Should scaling properties, etc., be included in appendix?}\snote{I think they should.}
    of the approximate $\fnorminf$ norm to obtain
    \begin{align*}
        \fnorminf(\qmat,\acc)
        &= \fnorminf(2 \difmat' + \mathbf{1} (\conceptb)^T,\acc) \\
        &\le \fnorminf(2 \difmat',\acc)
            + \fnorminf(\mathbf{1} (\conceptb)^T,0) \\
        &\le 2 \fnorminf(\difmat',\acc/2) + 1.\qedhere
    \end{align*}
\end{proof}



\section{\texorpdfstring{Derivation of dual formulation of $\eta$ norm}{Derivation of dual formulation of eta norm}}
\label{ap:realizable-lb}

\begin{proof}[Proof of Lemma~\ref{lm:realizable-dual}]

    Let
    $
        L_\concepts
        = \{ G \in \R^{\concepts \times (\iuni \times \oo)} : \fnorminf(G) \le t \}
    $.
    Let $K_\concepts$ and $K_\concepts'$ be as defined
    by equations \eqref{eq:qmatset} and \eqref{eq:qmatsetgen}.
    By definition, $\mcnew(\concepts,\acc) > t$
    if and only if $L_\concepts$ and $K_\concepts'$
    are disjoint.

    Given some $U \in \R^{\concepts \times (\iuni \times \oo)}$,
    we are interested in the quantities
    $
        \max \{ U \cdot G : G \in L_\concepts \}
    $
    and
    $
        \min \{ U \cdot G : G \in K_{\concepts}' \}
    $.
    In particular, by the hyperplane separation theorem,
    since $L_\concepts$ and $K_\concepts'$ are convex and $L_\concepts$ is also compact,
    they are disjoint exactly when there exists
    some $U \in \R^{\concepts \times (\iuni \times \oo)}$ such that
    \[
        \max \{ U \cdot G : G \in L_\concepts \}
        <
        \min \{ U \cdot G : G \in K_{\concepts}' \}.
    \]

    By definition,
    \[
        \max \{ U \cdot G : G \in L_\concepts \}
        = t \fnorminf^*(U).
    \]
    Also,
    \begin{align*}
        &\phantom{{}={}}\min \{ U \cdot G : G \in K_{\concepts}' \} \\
        &= \min_{G \in K_{\concepts}'} \sum_{\concept \in \concepts,\ielem \in \iuni}
            ( u_{\concept,(\ielem,c(\ielem))} g_{\concept,(\ielem,\concept(\ielem))} + u_{\concept,(\ielem,-\concept(\ielem))} g_{\concept,(\ielem,-\concept(\ielem))} ) \\
        &= \min_{\substack{G \in K_{\concepts} \\ \theta \in \R^\concepts} } \sum_{\concept \in \concepts,\ielem \in \iuni}
        ( u_{\concept,(\ielem,\concept(\ielem))} \cdot (g_{\concept,(\ielem,\concept(\ielem))} + \theta_\concept) + u_{\concept,(\ielem,-\concept(\ielem))} \cdot ( g_{\concept,(\ielem,-\concept(\ielem))} + \theta_\concept) ) \\
        &= \min_{G \in K_{\concepts}} \sum_{\concept \in \concepts,\ielem \in \iuni}
        ( u_{\concept,(\ielem,\concept(\ielem))} \cdot g_{\concept,(\ielem,\concept(\ielem))} + u_{\concept,(\ielem,-\concept(\ielem))} \cdot  g_{\concept,(\ielem,-\concept(\ielem))} ) \\
        &\phantom{{}={}}+ \min_{\theta \in \R^\concepts} \sum_{\concept \in \concepts}
        \theta_\concept \cdot \sum_{\ielem \in \iuni} ( u_{\concept,(\ielem,\concept(\ielem))} + u_{\concept,(\ielem,-\concept(\ielem))} )
    \end{align*}
    If, for some $\concept \in \concepts$, it holds that
    $\sum_{\ielem \in \iuni} ( u_{\concept,(\ielem,\concept(\ielem))} + u_{\concept,(\ielem,-\concept(\ielem))} ) \ne 0$,
    then 
    \[
        \min_{\theta_\concept \in \R}
        \theta_\concept \cdot \sum_{\ielem \in \iuni} ( u_{c,(\ielem,\concept(\ielem))} + u_{\concept,(\ielem,-\concept(\ielem))} )
        = - \infty.
    \]
    Also, if there exist $\concept \in \concepts$
    and $\elem \in \iuni$
    such that $u_{\concept,(\ielem,-\concept(\ielem))} < 0$,
    then
    \[
        \min_{G \in K_{\concepts}}
        u_{\concept,(\ielem,-\concept(\ielem))} g_{\concept,(\ielem,-\concept(\ielem))}
        = - \infty.
    \]
    However, in the remaining case where $U$ is in the set $S_\concepts$,
    then
    \begin{align*}
        \min \{ U \cdot G : G \in K_{\concepts}' \}
        &= \min_{G \in K_{\concepts}} \sum_{\concept \in \concepts,\ielem \in \iuni}
        ( u_{\concept,(\ielem,\concept(\ielem))} \cdot g_{\concept,(\ielem,\concept(\ielem))} + u_{\concept,(\ielem,-\concept(\ielem))} \cdot  g_{\concept,(\ielem,-\concept(\ielem))} ) \\
        &= \sum_{\concept \in \concepts,\ielem \in \iuni}
        ( - \acc | u_{\concept,(\ielem,c(\ielem))} | + u_{\concept,(\ielem,-\concept(\ielem))} ).
    \end{align*}
    With these facts at our disposal, we obtain
    \begin{align*}
        \mcnew(\concepts,\acc) > t 
        & \Leftrightarrow K_\concepts' \cap L_\concepts = \emptyset \\
        & \Leftrightarrow \exists U \in \R^{\concepts \times (\iuni \times \{\pm 1\})}, \
            \max \{ U \cdot G : G \in L_\concepts \}
            <
            \min \{ U \cdot G : G \in K_\concepts \} \\
        & \Leftrightarrow \exists U \in S_\concepts, \
            t \fnorminf^*(U)
            < \sum_{\concept \in \concepts,\ielem \in \iuni}
                ( - \acc | u_{\concept,(\ielem,c(\ielem))} | + u_{\concept,(\ielem,-\concept(\ielem))} ) \\
        & \Leftrightarrow
            \max_{U \in S_\concepts} \frac{ \sum_{\concept \in \concepts,\ielem \in \iuni}
                ( u_{\concept,(\ielem,-\concept(\ielem))} - \acc | u_{\concept,(\ielem,\concept(\ielem))} | ) }{\fnorminf^*(U)}
            > t
    \end{align*}
    Since the equivalence holds for all $t \in \R$,
    it follows that
    \begin{equation*}
        \mcnew(\concepts, \acc)
        = \max_{U \in S_\concepts} \frac
            {
                \sum_{\concept \in \concepts,\ielem \in \iuni}
                ( u_{\concept,(\ielem,-\concept(\ielem))} - \acc | u_{\concept,(\ielem,\concept(\ielem))} | ) 
            }
            {\fnorminf^*(U)}.
    \end{equation*}

\end{proof}


\cut{\section{Upper bound on agnostic learning in terms of approximate $\fnorminf$ norm of difference matrix}
\label{ap:difub}

\anote{The upper bound is not strictly necessary
since the result follows from \cite{factorization}
but it is nice for intuition.}
\snote{Definitely a good candidate for the appendix. Technically the upper bound is stronger than the one in \cite{factorization} because it gives the correct sample complexity 0 for a singleton concept class.}
\anote{It seems to give $O(1)$ sample complexity
now that we have accounted for uniform convergence.}

In this section, we present an upper bound
for agnostic learning under non-interactive LDP
in terms of the approximate $\fnorminf$ norm
of the difference matrix.
As a consequence of Lemma~\ref{lm:normrelation-ub}
this result implies the upper bound for agnostic learning
originally shown in \cite{factorization}.

\begin{theorem}

    Let $\concepts \subseteq \oo^\iuni$
    be a concept class with difference matrix $\difmat$
    as given by \eqref{eq:difmatdef}.
    Let $\priv > 0$,
    $\acc,\conf \in (0,1]$.
    Then there exists a non-interactive $\priv$-LDP protocol
    $
        \mech:(\iuni \times \oo)^\dsize \rightarrow \oo^\iuni
    $
    which $(\acc,\conf)$-learns $\concepts$ agnostically
    with $\dsize$ samples, where
    \[
        \dsize = O \left(
          \frac{(1+\fnorminf(\difmat,\acc/2))^2 \cdot \log(|\concepts|/\conf)}{\priv^2 \acc^2}
        \right).
    \]

\end{theorem}

\begin{proof}

    Let $\symdifmat$
    be the symmetrized difference matrix
    corresponding to $\concepts$
    and let
    $
        \{\query_{\concept,\conceptb}\}_{\concept,\conceptb \in \concepts}
    $
    be the corresponding workload of queries.
    Let $\dista$ be underlying distribution
    on $\iuni \times \oo$
    from which samples are drawn.

    Applying the approximate factorization mechanism
    of \cite{factorization},
    \tnote{This isn't the same matrix as in your earlier paper. Maybe Section 3.3 should be moved up. so that the upper bound proof is more clear.}
    the number of samples $\dsize$ required
    to estimate the value of
    $\query_{\concept,\conceptb}(\dista)$
    within $\acc/2$
    for all $\concept,\conceptb \in \concepts$,
    with probability of failure at most $\conf$
    is bounded by\snote{We also need uniform convergence to make sure that empirical losses and population losses are close enough. Does this bound account for that?}
    \[
        \dsize
        = O \left(
            \frac
            {(1+\fnorminf(\symdifmat,\acc/2))^2 \cdot \log(\qsize / \conf)}
                {\priv^2 \acc^2} 
        \right)
        = O \left(
            \frac
            {(1+\fnorminf(\difmat,\acc/2))^2 \cdot \log(\qsize / \conf)}
                {\priv^2 \acc^2} 
        \right).
    \]
    The bound in terms of $\fnorminf(\difmat,\acc/2)$
    is obtained as a consequence of Lemma~\ref{}\anote{
        Need to include this lemma somewhere.
    }
    which says that symmetrizing a matrix
    has no effect on its approximate $\fnorminf$ norm
    and hence
    $
        \fnorminf(\difmat,\acc/2)
        =
        \fnorminf(\symdifmat,\acc/2)
    $.

    For $\concept,\conceptb \in \concepts$,
    let $\qans_{\concept,\conceptb}$
    denote the estimated value of
    query $\query_{\concept,\conceptb}$
    so that
    $
        | \qans_{\concept,\conceptb}
        -
        \query_{\concept,\conceptb}(\dista) |
        \le \acc/2.
    $
    Let the output $\hyp \in \concepts$
    of our algorithm be
    the concept which minimizes
    $
        \max_{\conceptb \in \concepts} \qans_{\hyp,\conceptb}
    $.
    By \eqref{eq:queryvalues},
    \[
        \forall \concept \in \concepts, \ 
        \query_{\concept,\conceptb}(\dista)
        = \loss_{\dista}(\concept) - \loss_{\dista}(\conceptb).
    \]
    The concept $\concept \in \concepts$
    which minimizes $\loss_\dista(\concept)$
    satisfies
    \[
        \max_{\conceptb \in \concepts} 
        \query_{\concept,\conceptb}(\dista)
        =
        \max_{\conceptb \in \concepts} 
        \left(
            \loss_{\dista}(\concept) - \loss_{\dista}(\conceptb)
        \right)
        = 0.
    \]
    Hence, there must exist some
    $\concept \in \concepts$
    such that $\qans_{\concept,\conceptb} \le \acc/2$.
    In particular,
    $\max_{\conceptb \in \concepts} \qans_{\hyp,\conceptb} \le \acc/2$
    holds.
    It follows that
    $\max_{\conceptb \in \concepts} \query_{\hyp,\conceptb}(\dista) \le \acc$,
    which implies
    $
        \loss_{\dista}(\hyp)
        \le
        \min_{\conceptb \in \concepts}
            \loss_{\dista}(\conceptb) + \acc.
    $
\end{proof}

}


\end{document}